\documentclass{article}

\usepackage{microtype}
\usepackage{graphicx}
\usepackage{booktabs} 

\usepackage[accepted]{icml2025}

\usepackage{amsmath}
\usepackage{amssymb}
\usepackage{mathtools}
\usepackage{amsthm}

\usepackage[capitalize,noabbrev]{cleveref}

\theoremstyle{plain}
\newtheorem{theorem}{Theorem}[section]

\theoremstyle{definition}
\newtheorem{definition}[theorem]{Definition}

\theoremstyle{remark}

\usepackage[textsize=tiny]{todonotes}

\usepackage{amsmath}
\usepackage{graphicx}
\usepackage{booktabs}
\usepackage{multirow,multicol}
\usepackage{tikz}
\usepackage{color}
\usepackage{colortbl}
\usepackage{amsmath}
\usepackage{amssymb}
\usepackage{bm}
\usepackage{subfigure}
\usepackage{adjustbox}

\newcommand{\method}{\textsc{AMP}}

\newcommand{\rv}{\textit{r.v.\@}}

\newcommand{\quotes}[1]{``#1''}

\definecolor{rwth-red}{cmyk}{.15,1,1,0}\colorlet{rwth-lred}{rwth-red!50}\colorlet{rwth-llred}{rwth-red!25}
\definecolor{rwth-green}{cmyk}{.7,0,1,0}\colorlet{rwth-lgreen}{rwth-green!50}\colorlet{rwth-llgreen}{rwth-green!25}

\icmltitlerunning{Adaptive Message passing}

\begin{document}

\twocolumn[
\icmltitle{Adaptive Message Passing: A General Framework to Mitigate \\ Oversmoothing, Oversquashing, and Underreaching}



\icmlsetsymbol{equal}{*}

\begin{icmlauthorlist}
\icmlauthor{Federico Errica}{nec}
\icmlauthor{Henrik Christiansen}{nec}
\icmlauthor{Viktor Zaverkin}{nec} \\
\icmlauthor{Takashi Maruyama}{nec}
\icmlauthor{Mathias Niepert}{nec,stuttgart}
\icmlauthor{Francesco Alesiani}{nec}
\end{icmlauthorlist}

\icmlaffiliation{nec}{NEC Laboratories Europe}
\icmlaffiliation{stuttgart}{University of Stuttgart}

\icmlcorrespondingauthor{Federico Errica}{federico.errica@neclab.eu}

\icmlkeywords{Graph Machine Learning, Variational Inference, Oversmoothing, Oversquashing, Underreaching, Depth Learning}

\vskip 0.3in
]



\printAffiliationsAndNotice{}  

\begin{abstract}
    Long-range interactions are essential for the correct description of complex systems in many scientific fields. The price to pay for including them in the calculations, however, is a dramatic increase in the overall computational costs. Recently, deep graph networks have been employed as efficient, data-driven models for predicting properties of complex systems represented as graphs. These models rely on a message passing strategy that should, in principle, capture long-range information without explicitly modeling the corresponding interactions. In practice, most deep graph networks cannot really model long-range dependencies due to the intrinsic limitations of (synchronous) message passing, namely oversmoothing, oversquashing, and underreaching. This work proposes a general framework that \textit{learns to mitigate} these limitations: within a variational inference framework, we endow message passing architectures with the ability to adapt their depth and filter messages along the way. With theoretical and empirical arguments, we show that this strategy better captures long-range interactions, by competing with the state of the art on five node and graph prediction datasets.
\end{abstract}

\section{Introduction}
\label{sec:introduction}
Complex systems, characterized by interacting entities and emergent behavior, are a cornerstone of research in many scientific disciplines. Mathematical models of such systems should consider the effects of both short and long-range interactions between entities, and the latter are often crucial to describe the system's behavior with the highest degree of precision. 
For instance, in computational physics, it is well-known that electrostatic and gravitational interactions decay slowly with distance \citep{campa_physics_2014};
in computational chemistry and material sciences, the accurate modeling of non-local effects, such as non-bonded interactions in molecular systems, is necessary to estimate properties like the free energy \citep{shirts_accurate_2007,piana_evaluating_2012};
in biology, disrupting long-range interactions in mRNA can inhibit slicing \citep{ruegsegger_block_2001}; 
in immunology, the distant interactions between a major histocompatibility complex and regions of the T-cell receptor molecule correlate with their binding process \citep{ferber_tcell_2012}.

Complex systems can be represented as graphs of interacting entities. Modeling long-range interactions often implies that the graph has dense connectivity, meaning the number of interactions is quadratic in the number of entities. Machine learning researchers tried to address these problems by relying on accurate surrogates for computationally demanding simulations \citep{sanchez_learning_2020}. Some of these methods rely on Deep Graph Networks (DGNs) \citep{bacciu_gentle_2020}, deep learning models implement a \textit{message passing} paradigm of computation. In message passing, nodes repeatedly exchange messages with each other to propagate information across the graph and compute their embeddings. More rounds of message passing increase the ``receptive field" of each node.

Despite its long-standing history \citep{sperduti_supervised_1997,micheli_neural_2009,scarselli_graph_2009,bacciu_gentle_2020}, research in graph representation learning has gained more traction in recent years, and there are still many open questions. For instance, it is well-known that most message passing architectures are ineffective at capturing long-range dependencies, thus reducing their impact in the scientific fields mentioned before. Researchers relate this problem to at least three others, namely \textit{oversmoothing} \citep{li_deeper_2018}, \textit{oversquashing} \citep{alon_on_2021,rusch_survey_2023}, and \textit{underreaching} \citep{alon_on_2021}. Briefly, oversmoothing means that the node embeddings of a DGN tend to converge to the same value as the depth increases. In contrast, oversquashing relates to the bottleneck of compressing a (possibly) exponential amount of information from neighboring nodes into a single node embedding. Finally, underreaching refers to DGNs' inability to propagate a node's information to more than $K$ hops away.

This work provides a general framework for improving the ability of \textit{any} message passing architecture to capture long-range dependencies; we extend the general message passing formulation to propagate relevant information across the graph. At the heart of our proposal is the idea to let DGNs learn \textit{how many layers} of message passing to use and \textit{when to send} specific messages. As a matter of fact, one typically observes oversmoothing and oversquashing when too many messages are propagated, hence learning which messages to discard is important. At the same time, solving underreaching requires a sufficient number of message passing rounds to be performed, and it is crucial to learn this information from that task rather than guessing it via expensive grid searches. In light of these characteristics, we call our approach Adaptive Message Passing (\method{}).

Our contributions are multi-faceted. We extend a recent variational framework for unbounded depth networks \citep{nazaret_variational_2022} to the processing of graphs, and we introduce new families of distributions with specific properties to overcome previous limitations. We also propose a soft message filtering scheme to prune irrelevant information for the task at hand and favor the propagation of messages to distant regions of the input graph. Theoretically, we show how to propagate a message unchanged between any two connected nodes in the graph and that filtering decreases the upper bound on sensitivity \citep{topping_understanding_2022}; thus, underreaching and oversquashing can be mitigated. 
Empirically, \method{} significantly and consistently improves the performances of message passing architectures on five well-known node and graph prediction datasets where long-range information is important. Qualitative analyses provide further evidence that \method{} mitigates oversmoothing and oversquashing as well as learning the number of message-passing layers. Finally, we conduct an in-depth study of our approach via ablations and visualizations of the models' predictions.
\section{Related Work}
\label{sec:related-works}
Due to space reasons, we provide an extended literature review in Appendix \ref{sec:extended-related work}.

\vspace{-1em}
\paragraph{Oversquashing.}
There are many methods that attempt to address the oversquashing problem with the goal of better capturing long-range dependencies \citep{alon_on_2021}. Some works learn how a node should completely stop propagating a message in a fixed-depth architecture \citep{spinelli_adaptive_2020} or if it should only listen, isolate, or receive/broadcast its own message \citep{finkelshtein_cooperative_2023}. Similarly, one can learn to sample edges at each message passing layer according to some learned parametrization \citep{hasanzadeh_bayesian_2020} or have a completely asynchronous message passing \citep{faber_asynchronous_2023}. Ordinary differential equation-based message passing approaches, instead, provably preserve information regardless of the depth in the network \citep{gravina_antisymmetric_2023,heilig2025porthamiltonian}.
On the other hand, most graph rewiring methods alter the graph connectivity to increase the node sensitivity \citep{topping_understanding_2022}, which has been theoretically linked to a form of ``topological oversquashing''. Some of them preserve locality and sparsity of the rewiring process \citep{barbero_locality_2023} or perform dynamic rewiring \citep{gutteridge_drew_2023}. Probabilistic approaches to rewiring are based on sampled sub-graphs \citep{qian_probabilistically_2023}. Recently, a critical perspective on the effectiveness of rewiring approaches has also been given \citep{tortorella_leave_2022}.

\paragraph{Oversmoothing.}
One practical mitigation to the oversmoothing problem is dropping edges, reducing the overall flow of messages and avoiding the convergence of all embeddings to the same value \citep{rong_dropedge_2020}. Another possibility is to employ skip/residual connections \citep{kipf_semisupervised_2017,li_deepgcns_2019}. The concatenation of node representations across layers is yet another a way to contain oversmoothing, which has been adopted in neural and probabilistic models to improve the downstream performances on several node and graph-related tasks \citep{bacciu_contextual_2018,xu_representation_2018}.
Instead, an orthogonal research direction considers implicit neural networks for graphs that correspond to infinite-depth models and seem to be able to capture long-range dependencies \citep{poli_continuous_2021,liu_eignn_2021}. These models simulate synchronous message passing with a potentially infinite number of message-propagation steps, and some of them appear to be empirically robust to the oversmoothing problem.

\paragraph{Adaptive Architectures.}
This family of methods tries to learn the models' architecture during training. Our work is inspired by the unbounded depth networks (UDNs) of \citet{nazaret_variational_2022}, who proposed a variational framework for learning the depth in multi-layer perceptrons and convolutional neural networks. In the graph domain, the first approach to learning the depth of a DGN was proposed by \citet{micheli_neural_2009}, who applied the cascade correlation algorithm \citep{fahlman_cascade_1989}. Other works attempted to learn the ``width" of each message passing layer by exploiting Bayesian non-parametric models \citep{castellana_infinite_2022}, which allows to save time and memory when building deeper probabilistic DGNs.
Finally, it is important to notice that these works, including this manuscript, are all orthogonal to the popular field of neural architecture search \citep{zhou_auto_2022}: The former attempts at dynamically modifying the architecture \textit{during learning}, whereas most neural architecture search approaches find smarter ways to carry out a grid search or assume a maximum number of layers. An advantage of adaptive approaches is that they can greatly reduce time and computational costs to perform a hyper-parameter search.

Contrary to all these approaches, AMP uses a variational framework to jointly learn both the depth of the DGN and filter messages passed between nodes in each of these layers.

\section{Adaptive Message Passing}
\label{sec:method}
This section introduces the probabilistic framework of \method{}, which we train with simple backpropagation to optimize a variational lower bound. This bound is composed of a predictive term and two (optional) regularizers over the parameters and the DGN's depth.

\paragraph{Definitions.} We consider directed attributed graphs $g=(\mathcal{V}, \mathcal{E}, \mathcal{X}, \mathcal{A})$, each consisting of a set of nodes $\mathcal{V}=\{1,\dots,n_g\}$ that are connected together via a set of oriented edges $\mathcal{E}=\{(u,v) | u,v \in \mathcal{V}\}$. When a graph is undirected, each edge is converted into two oriented ones, that is, $(u,v)$ and $(v,u)$. The set $\mathcal{X}=\{\mathbf{x}_v \in \mathbb{R}^d | v \in \mathcal{V}\}$ defines the $d$-dimensional attribute vector of each node in the graph, and similarly for the $d'$-dimensional edge attributes belonging to the set $\mathcal{A}=\{\bm{a}_{uv} \in \mathbb{R}^{d'} | (u,v) \in \mathcal{E}\}$. Finally, we define the neighborhood of a node $v$ as the set of incoming edges $\mathcal{N}_v = \{ u | (u,v) \in \mathcal{E}\}$. As outlined in previous works \citep{hammer_universal_2005,bongini_recursive_2018}, each attributed graph can be seen as a realization of \textit{some} random variable (\rv{}) $\mathcal{G}$ with support in the graph domain. Similarly to classical machine learning, we do not have access to the data distribution $p(\mathcal{G})$, rather we are interested in modeling the conditional distribution $p(\mathcal{T}=Y_g | \mathcal{G}=g)$, where $Y_g$ stands for the target value(s) to be predicted.

\paragraph{Multi-output Family of Architectures.} \method{} produces deep graph networks of \textit{potentially infinite depth}, where \textbf{each layer} $\ell$ comprises a message passing operation MP \textbf{and} a readout mapping $\mathcal{R}$ \citep{bacciu_gentle_2020} from node embeddings to the desired output. 

Without loss of generality, a message passing layer $\text{MP}_\ell$ can compute node embeddings $\bm{h}^{\ell}_v, \forall v \in \mathcal{V}$ as follows:
\begin{align}
\bm{h}_v^{\ell} = \phi^{\ell} \left(\bm{h}_v^{\ell-1},\ \Psi(\{\psi^{\ell}(\bm{h}_u^{\ell-1}, \bm{a}_{uv}) | u \in \mathcal{N}_v\} ) \right),
\label{eq:message passing-aggregation}
\end{align}
where $\phi^{\ell}$ and $ \psi^{\ell}$ are learnable functions and $\Psi$ is a permutation invariant function that aggregates the embeddings of $v$'s neighbors computed at the previous layer. When $\ell = 1$, $\bm{h}_v^{1}$ is obtained by applying a learnable transformation of the node $v$'s features $\bm{h}_v^{0}=\bm{x}_v$, and no neighbor aggregation is performed. Instead, the readout mapping $\mathcal{R}^\ell$ depends on the task: If one needs node-wise predictions, then the readout implements a learnable map $\hat{\bm{y}}_v^{\ell} = \rho^{\ell}(\bm{h}_v^{\ell})$ from $\bm{h}_v^{\ell}$ to a node output $\hat{\bm{y}}_v^{\ell}$; On the other hand, in the case of whole-graph predictions, a global aggregation has to be performed first:
\begin{align}
    \hat{\bm{y}}^{\ell} = \rho_2^{\ell}\left(\Phi\left(\{\rho_1^{\ell}(\bm{h}_v^{\ell}) | v \in \mathcal{V} \}\right)\right),
\end{align}
where $\rho_1,\rho_2$ denote learnable functions and $\Phi$ is a global pooling function that aggregates all node representations computed at a given layer $\ell$. 
The learnable functions $\phi^{\ell}, \psi^{\ell}, \rho_1^{\ell},\rho_2^{\ell}$ are typically implemented as 1-hidden layer MLPs parametrized by $\Theta_{\ell}$.

\begin{figure}[t]
    \centering
    \resizebox{0.8\columnwidth}{!}{\tikzset{every picture/.style={line width=0.75pt}} 

\begin{tikzpicture}[x=0.75pt,y=0.75pt,yscale=-1,xscale=1]

\draw [line width=1.5]    (86.33,59.92) .. controls (136.02,21.88) and (232.06,19.34) .. (286.71,41.84) ;
\draw [shift={(290,43.25)}, rotate = 203.96] [fill={rgb, 255:red, 0; green, 0; blue, 0 }  ][line width=0.08]  [draw opacity=0] (11.61,-5.58) -- (0,0) -- (11.61,5.58) -- cycle    ;
\draw [line width=1.5]    (424.33,58.92) -- (338,59.56) ;
\draw [shift={(334,59.59)}, rotate = 359.58] [fill={rgb, 255:red, 0; green, 0; blue, 0 }  ][line width=0.08]  [draw opacity=0] (11.61,-5.58) -- (0,0) -- (11.61,5.58) -- cycle    ;
\draw [line width=1.5]    (309,59.59) -- (223.58,140.14) ;
\draw [shift={(220.67,142.89)}, rotate = 316.68] [fill={rgb, 255:red, 0; green, 0; blue, 0 }  ][line width=0.08]  [draw opacity=0] (11.61,-5.58) -- (0,0) -- (11.61,5.58) -- cycle    ;
\draw [line width=1.5]    (424.33,58.92) -- (342.83,97.25) -- (232.94,149.83) ;
\draw [shift={(229.33,151.56)}, rotate = 334.43] [fill={rgb, 255:red, 0; green, 0; blue, 0 }  ][line width=0.08]  [draw opacity=0] (11.61,-5.58) -- (0,0) -- (11.61,5.58) -- cycle    ;
\draw [line width=1.5]    (86.33,59.92) -- (185.43,142.69) ;
\draw [shift={(188.5,145.25)}, rotate = 219.87] [fill={rgb, 255:red, 0; green, 0; blue, 0 }  ][line width=0.08]  [draw opacity=0] (11.61,-5.58) -- (0,0) -- (11.61,5.58) -- cycle    ;
\draw [line width=1.5]    (205,59.92) -- (280,59.61) ;
\draw [shift={(284,59.59)}, rotate = 179.76] [fill={rgb, 255:red, 0; green, 0; blue, 0 }  ][line width=0.08]  [draw opacity=0] (11.61,-5.58) -- (0,0) -- (11.61,5.58) -- cycle    ;
\draw [line width=1.5]    (205,59.92) -- (205.95,134.92) ;
\draw [shift={(206,138.92)}, rotate = 269.27] [fill={rgb, 255:red, 0; green, 0; blue, 0 }  ][line width=0.08]  [draw opacity=0] (11.61,-5.58) -- (0,0) -- (11.61,5.58) -- cycle    ;
\draw  [fill={rgb, 255:red, 255; green, 255; blue, 255 }  ,fill opacity=1 ][line width=1.5]  (60.33,59.92) .. controls (60.33,46.12) and (71.53,34.92) .. (85.33,34.92) .. controls (99.14,34.92) and (110.33,46.12) .. (110.33,59.92) .. controls (110.33,73.73) and (99.14,84.92) .. (85.33,84.92) .. controls (71.53,84.92) and (60.33,73.73) .. (60.33,59.92) -- cycle ;
\draw  [fill={rgb, 255:red, 31; green, 119; blue, 180 }  ,fill opacity=1 ][line width=1.5]  (181,163.92) .. controls (181,150.12) and (192.19,138.92) .. (206,138.92) .. controls (219.81,138.92) and (231,150.12) .. (231,163.92) .. controls (231,177.73) and (219.81,188.92) .. (206,188.92) .. controls (192.19,188.92) and (181,177.73) .. (181,163.92) -- cycle ;
\draw  [fill={rgb, 255:red, 31; green, 119; blue, 180 }  ,fill opacity=1 ][line width=1.5]  (180,59.92) .. controls (180,46.12) and (191.19,34.92) .. (205,34.92) .. controls (218.81,34.92) and (230,46.12) .. (230,59.92) .. controls (230,73.73) and (218.81,84.92) .. (205,84.92) .. controls (191.19,84.92) and (180,73.73) .. (180,59.92) -- cycle ;
\draw  [fill={rgb, 255:red, 255; green, 255; blue, 255 }  ,fill opacity=1 ][line width=1.5]  (399.33,58.92) .. controls (399.33,45.12) and (410.53,33.92) .. (424.33,33.92) .. controls (438.14,33.92) and (449.33,45.12) .. (449.33,58.92) .. controls (449.33,72.73) and (438.14,83.92) .. (424.33,83.92) .. controls (410.53,83.92) and (399.33,72.73) .. (399.33,58.92) -- cycle ;
\draw  [fill={rgb, 255:red, 255; green, 255; blue, 255 }  ,fill opacity=1 ][line width=1.5]  (284,59.59) .. controls (284,45.78) and (295.19,34.59) .. (309,34.59) .. controls (322.81,34.59) and (334,45.78) .. (334,59.59) .. controls (334,73.4) and (322.81,84.59) .. (309,84.59) .. controls (295.19,84.59) and (284,73.4) .. (284,59.59) -- cycle ;
\draw  [line width=1.5]  (164,32) .. controls (164,22.06) and (172.06,14) .. (182,14) -- (345.33,14) .. controls (355.27,14) and (363.33,22.06) .. (363.33,32) -- (363.33,179.89) .. controls (363.33,189.83) and (355.27,197.89) .. (345.33,197.89) -- (182,197.89) .. controls (172.06,197.89) and (164,189.83) .. (164,179.89) -- cycle ;
\draw  [line width=1.5]  (13,28.02) .. controls (13,21.93) and (17.93,17) .. (24.02,17) -- (139.4,17) .. controls (145.49,17) and (150.42,21.93) .. (150.42,28.02) -- (150.42,118.54) .. controls (150.42,124.62) and (145.49,129.56) .. (139.4,129.56) -- (24.02,129.56) .. controls (17.93,129.56) and (13,124.62) .. (13,118.54) -- cycle ;

\draw (85.33,59.92) node  [font=\Large]  {$\Theta _{\ell }$};
\draw (206,163.92) node  [font=\Large,color={rgb, 255:red, 255; green, 255; blue, 255 }  ,opacity=1 ]  {$\mathcal{T}_{i}$};
\draw (301.76,183.92) node  [font=\Large] [align=left] {$\displaystyle i=1\dotsc |\mathcal{D} |$};
\draw (204,59.92) node  [font=\Large,color={rgb, 255:red, 255; green, 255; blue, 255 }  ,opacity=1 ]  {$\mathcal{G}_{i}$};
\draw (424.33,58.92) node  [font=\Large]  {$\mathcal{L}$};
\draw (309,59.59) node  [font=\Large]  {$\mathcal{F}_{i}$};
\draw (69.97,115.01) node  [font=\Large] [align=left] {$\displaystyle \ell =1\dotsc \infty $};

\end{tikzpicture}}
    \caption{The graphical model of \method{}, where white and blue circles denote, respectively, latent and observable random variables. 
    $\Theta_{\ell}$ is the \rv{} over the parameters of layer $k$, $\mathcal{F}_i$ defines a distribution over the message filters, $\mathcal{L}$ implements a distribution over the layers of the architecture, while $\mathcal{G}_i$ and $\mathcal{T}_i$ are distributions over the (observable) input graph and the target label, respectively.}
    \label{fig:graphical-model}
\end{figure}

\paragraph{Variational Inference}
Given a dataset $\mathcal{D}$ of $|\mathcal{D}|$ \textit{i.i.d.}\ graphs, we seek to maximize the log-likelihood
\begin{align}
    \ln p(Y|G)=\ln\prod_i^{|\mathcal{D}|}p(Y_{g_i}|g_i)=\sum_i^{|\mathcal{D}|}\ln p(Y_{g_i}|g_i).
\end{align}
One usually assumes the existence of a set of latent variables $Z$ such that $p(Y|G)=\int p(Y,Z|G)dZ$. By designing a proper graphical model, which encodes conditional independence assumptions, we define how to compute $p(Y,Z|G)$; however, the integral to maximize often remains intractable. Therefore, one approach is to turn to variational inference \citep{jordan_introduction_1999,blei_variational_2017}, which maximizes the Expected Lower Bound (ELBO) instead. In particular, by arbitrarily defining a distribution $q(Z|G)$, one has
\begin{align}
    \ln p(Y|G) \geq \mathbb{E}_{q(Z|G)}\left[\ln \frac{p(Y,Z| G)}{q(Z| G)}\right].
\end{align}
Below, we show how we define the joint distribution $p(Y,Z|G)$ as well as the variational distribution $q(Z|G)$ in the specific context of \method{}.

\paragraph{\method{} Formulation.} Figure \ref{fig:graphical-model} represents the graphical model associated with \method{}, where white and blue circles represent latent and observed \rv{}s, respectively. We extend the formulation of \citet{nazaret_variational_2022} to the domain of graphs by modeling the message filtering strategy: The variable $\Theta_{\ell}$ follows a distribution over the parameters of layer $\ell$ of an infinite-depth network, and $\mathcal{L}$ follows a distribution over layers $L \in \mathbb{N}^*$ and is used to truncate the network to a finite depth $L$. For the  $i$-th graph $g_i$, $\mathcal{F}_i$ follows a distribution over soft message filters $\bm{F}_i \in [0,1]^{|\mathcal{V}| \times L \times d}$. In particular, given a node $v$ and a layer $\ell$, the $d$-dimensional vector $\bm{F}_i(v,\ell)$ specifies how much of $h^{\ell}_v$ has to be propagated through the outgoing edges in the next message passing layer. The generative model is
\begin{align}
& \bm{\theta} \sim p(\Theta) = \prod^{\infty}_{\ell=1}p(\Theta_{\ell}) \\    
& L \sim p(\mathcal{L}) \\    
& \bm{F}_i | \mathcal{}g_i, L, \bm{\theta} \sim p(\mathcal{F}_i | \mathcal{}g_i, L, \bm{\theta}) \\    
& Y_{g_i} | g_i, \bm{F}_i, L, \bm{\theta} \sim p(\mathcal{T}_i; \ \Omega_L(g_i, \bm{F}_i, \bm{\theta})),
\end{align}
with $\Omega_L$ being the infinite DGN truncated at depth $L$ whose output parametrizes the target distribution. This means that the joint distribution decomposes as 
\begin{align}
    & p(Y,\bm{\theta},\bm{F}_i,L|G)= \nonumber \\ & = p(\bm{\theta})p(L)\prod_i^{|\mathcal{D}|}p(Y_{g_i}|\bm{\theta},\bm{F}_i,L,g_i)p(\mathcal{F}_i | \mathcal{}g_i, L, \bm{\theta})
\end{align}
Note that the independence of the priors is key for an efficient approximation of the posterior distribution \cite{nazaret_variational_2022}.

In Figure \ref{fig:overview}, we visually represent the effect that message filtering has on the propagation of messages across DGN layers. A graph of seven nodes (a) is provided and the message filtering scheme (b) has been discretized in the interest of simplicity. For instance, node $1$ will send its message only at message passing layer $1$, nodes $2$ \& $3$ will never send a message, and node $4$ will send a message only at layer $2$. Compared to the standard message passing (c), where all nodes send their messages at each layer, \method{} implements a learnable filtering (d), where a subset of all possible messages is propagated at each layer in a way that depends on the task to be solved. In Section \ref{sec:discussion-oversquashing-underreaching-oversmoothing}, we discuss the implications of this adaptive message filtering scheme in mitigating the well-known issues of oversquashing, underreaching, and oversmoothing. Notably, message filtering does not introduce a significant computational burden since it has linear complexity in the number of nodes.

\begin{figure}[t]
    \centering
    \resizebox{1\columnwidth}{!}{\input{images/overview}}
    \caption{Given an input graph (a) and a discrete message filtering scheme (b), we observe how a $L$=2-layer standard message passing (c) differs from \method{} (d) in terms of the number of messages sent. Please refer to the text for more details.} 
    \label{fig:overview}
\end{figure}

\paragraph{Choice of the Variational Distributions.} 
We now need to define the learnable variational distribution $q(\bm{\theta}, L, \bm{F}_i, | g_i, Y_{g_i})$. We assume it factorizes as $q(\bm{\theta} | L ; \bm{\nu})q(L ; \bm{\lambda})q(\bm{F}_i | g_i, L, \bm{\theta})$, where $\bm{\nu}$, $\bm{\lambda}$ are learnable parameters. We also assume that the variational posterior does not depend on $Y_{g_i}$ (so we can drop the term) to allow for predictions on unseen graphs. Below, we describe how to compute each factor so that the computation of the ELBO is tractable.

The distribution $q(L ; \bm{\lambda})$ has to belong to an \textit{unbounded with bounded and connected members'} family (see Definition \ref{def:truncated-family} in Section \ref{sec:extension-truncated}). In short, since the support of each distribution $q$ in the family is bounded, we can compute its expectation $\mathbb{E}_{q(L ; \bm{\lambda})}[f(L)]$ as the sum $\sum_{\ell \in \mathrm{support}(q)}q(\ell)f(\ell)$ for any function $f$. In Appendix \ref{sec:extension-truncated}, we extend the original treatment of Poisson distributions to Gaussians and mixtures of distributions.
Second, we define $q(\bm{\theta} | L ; \bm{\nu})$ such that we cannot make any statement about the layers greater than $L$ \citep{kurihara_collapsed_2007}:
\begin{align}
    q(\bm{\theta} | L ; \bm{\nu}) = q(\bm{\theta}_{1:L} ; \bm{\nu}_{1:L}) \prod_{\ell=L+1}^{\infty}p(\bm{\theta}_\ell)
\end{align}
and $p(\bm{\theta}_{\ell})$ can be, for instance, a Gaussian prior. We also fix $q(\bm{\theta}_{1:L} ; \bm{\nu}_{1:L}) = \prod_{\ell=1}^L \mathcal{N}(\bm{\theta}_{\ell} ; \bm{\nu}_{\ell}, \bm{I})$.

Finally, we define $q(\bm{F}_i | g_i, L, \bm{\theta})$ as a Dirac delta function $\delta_{\bm{F}_i}$ whose parameters $\bm{F}_i \in [0,1]^{|\mathcal{V}| \times L \times d}$ are computed by a function $f(g_i)$. Choosing the delta function makes the computation of its expectation straightforward, but other choices can, in principle, be made. We propose two versions of the function $f(g_i)$, whose choice is left as a hyper-parameter: the first computes $\bm{F}_i(v,\ell) = f_{\ell}(\bm{x}_v)$ and the second computes $\bm{F}_i(v,\ell) = f_{\ell}(\bm{h}^{\ell}_v)$, where $f_{\ell}$ is a Multi-Layer Perceptron (MLP) with sigmoidal activations. In other words, a node's outgoing messages will be filtered according to either the input features of that node or its embedding at layer $\ell$. Given $\bm{F}_i \sim q(\bm{F}_i | g_i, L, \bm{\theta})$, we extend Equation \ref{eq:message passing-aggregation} to apply such filtering:
\begin{align}
\bm{h}_v^{\ell} & = \phi^{\ell} (\bm{h}_v^{\ell-1}, \nonumber \\
& \Psi(\{\bm{F}_i(u,\ell-1) \odot \psi^{\ell}(\bm{h}_u^{\ell-1}, \bm{a}_{uv}) | u \in \mathcal{N}_v\} ) ),
\label{eq:filtered-message passing-aggregation}
\end{align}
with $\odot$ being the element-wise product. Such message filtering is similar in spirit to many works \citep{franceschi_learning_2019,spinelli_adaptive_2020,finkelshtein_cooperative_2023}, but our approach does not require gradient approximations caused by discrete operations and is fully differentiable. 

\paragraph{Computation of the ELBO.} Our choice of the variational distributions allows us to compute the ELBO efficiently and maximize it using backpropagation \citep{rumelhart_learning_1986}. In particular, we write (the full derivation is in Appendix \ref{sec:elbo-derivation})
\begin{align}
& \ln p(Y|G) \ge \sum_{\ell=1}^{\hat{L}}q(\ell ; \bm{\lambda})\Biggl[\ln\frac{p(\ell)}{q(\ell ; \bm{\lambda})} + \ln\frac{p(\bm{\nu})}{q(\bm{\nu} | \ell ; \bm{\nu})} \nonumber \\ 
& + \sum_i^{|\mathcal{D}|}\left[\ln\frac{p(\bm{F}_i)}{q(\bm{F}_i | g_i, \ell, \bm{\nu})} + \ln p(Y_{g_i} | \ell, \bm{F}_i, \bm{\nu}, g_i)\right]\Biggr],
\end{align}
where $\hat{L}=\mathrm{support}(q(L))$, $p(L)$ is a prior over layers, such as a Poisson distribution, and $p(\bm{F}_i)$ is a prior over all possible message filtering schemes (\textit{uninformative in this work}, so the term cancels). The second equivalence relies on the specific properties of the variational distributions and on the approximation of expectation $\mathbb{E}_{q(\bm{\theta} | L ; \bm{\nu})}[f(\bm{\theta}_{1:L})])$, for a function $f$, at the first order\footnote{First-order second-moment method (FOSM) of probability.} with $f(\mathbb{E}_{q(\bm{\theta} | L ; \bm{\nu})}[\bm{\theta}_{1:L}]) = f(\bm{\nu}_{1:L})$ as in \citet{nazaret_variational_2022}.

Akin to \citet{nazaret_variational_2022}, \method{} makes predictions about a new graph $g_j$ as:
\begin{align}
    p(Y_{g_j} | g_j) & \approx \mathbb{E}_{q(\bm{\theta}, L, \bm{F}_j, | g_j)}\left[p(Y_{g_j} ; \Omega_L(g_j, \bm{F}_j, \bm{\theta}) \right] \\
    & \approx \sum_{\ell=1}^{\hat{L}}q(\ell ; \bm{\lambda}) p(Y_{g_j} ; \Omega_{\ell}(g_j, \bm{F}_j, \bm{\nu})).
\end{align}
In other words, using the fact that $q(\ell;\bm{\lambda})$ has bounded support up to $\hat{L}$, we obtain the prediction as the weighted sum of the $\hat{L}$ output layers of the DGN, and the variational distribution $q(\ell;\bm{\lambda})$ over layers provides said weights. 

We now show that there is a direct relation between the ability to filter out messages of Equation \ref{eq:filtered-message passing-aggregation} and the upper bound on the Jacobian sensitivity discussed in \citet{di_over_2023}.
\begin{theorem}
    \label{th:upper-bound}
    For \method{} with $m$ layers and $u, v \in \mathcal{V}$,    
    \[
    \left\| \frac{\partial h_v^{(m)}}{\partial h_u^{(0)}} \right\|_{L^1} 
    \leq d \left( \left( c_{\mathrm{up}} \left( c_{\mathrm{rs}} I + c_{\mathrm{mp}} \left( c_F k_h + k_F \right) A \right) \right)^m \right)_{vu}
    \]
    Here, MPNN is in the following form 
    \[
    h_v^{\ell} = \mathrm{up}\left( \mathrm{rs}(h_v^{\ell-1}) + \mathrm{mp}(\sum_u A_{vu} F(h_u^{\ell-1}) \odot h_u^{\ell-1}) \right)
    \]
    where $\mathrm{up}, \mathrm{rs},$ and $\mathrm{mp}$ are Lipschitz functions as in \citet{di_over_2023} with constants $c_{\mathrm{up}}, c_{\mathrm{rs}}, c_{\mathrm{mp}}$, $c_F$ is the upper bound of the entry-wise $L^ 1$ matrix norm of $\frac{\partial F}{\partial x}$ for the filtering function $F$, $k_h$ is the maximal absolute value among entries of $h$, and similarly $k_F$ for the output of $F$. 
\end{theorem}
\begin{proof}
    The proof is provided in Appendix \ref{sec:discussion-oversquashing-underreaching-oversmoothing}.
\end{proof}

Note that $k_F \leq 1$; if we consider for simplicity a constant filtering function, namely $c_F=0$ and we filter enough, meaning $k_F<1$, then filtering will decrease the sensitivity’s upper bound. At the same time, this helps to reduce the amount of messages that get squashed into a fixed size vector \citet{alon_on_2021}, contradicting the widely accepted notion that “improving sensitivity mitigates oversquashing”. Please consult Appendix \ref{sec:discussion-oversquashing-underreaching-oversmoothing} for a more detailed discussion on this matter.

\paragraph{Practical Considerations.}
The depth of \method{} varies dynamically; in particular, the support of the distribution $q(\ell ; \bm{\lambda})$ is obtained by truncating it as the quantile function evaluated at $0.99$ (in our experiments). Whenever the quantile threshold shifts, we either grow or shrink the DGN by instantiating a new message-passing layer and increasing the output dimension of the function $f(g_i)$ that produces $\bm{F}_i$. When shrinking the DGN, we can retain the excess layers to account for future expansions or delete them; here we opt for the retention strategy. Importantly, the depth is not a hyper-parameter to be tuned anymore. 
While \method{} requires choosing a family of truncated distributions $q(L)$ and a proper initialization, it is generally believed that this has a smaller effect on the final result \citep{goel_information_1981,bernardo_bayesian_2009}. 
Also, setting uninformative priors works well in our experiments but they are convenient way to penalize the computational costs of deeper networks.

We conclude with a theorem on \method{}'s ability to propagate a message unchanged from two connected nodes in a graph, which would not be possible on classical (synchronous) message-passing neural networks. However, achieving such behavior in practice might be difficult.

\begin{theorem}[Short Version]
For each graph $g$, a source node $v$ and a destination node $u$, there exists a parametrization of AMP and a depth $K$ such that $\bm{h}_{u}^{K} = \bm{h}_{v}^{0}=\bm{x}_v$.
\label{th:reachability}    
\end{theorem}
\vspace{-2.5em}
\begin{proof}
The proof is provided in Appendix \ref{sec:discussion-oversquashing-underreaching-oversmoothing}, and Figure \ref{fig:theorem-sketch} sketches the process formalized in the theorem.
\end{proof}

\subsection{Computational Considerations}
The cost of filtering messages is $\mathcal{O}(|\mathcal{V}|)$. Therefore, the message passing operation is not altered significantly, since it has a cost of $\mathcal{O}(|\mathcal{V}| + |\mathcal{E}|)$. However, the additional burden introduced by \method{}, compared to classical message-passing architectures, is the layer-wise readout that we implemented as an MLP. Classical MPNNs employ a single readout with cost $\mathcal{O}(|\mathcal{V}|)$ or $\mathcal{O}(1)$ depending on the task nature, whereas we use one per layer, so we have $\mathcal{O}(|\mathcal{V}|L)$ and $\mathcal{O}(|\mathcal{V}|L)$ respectively. In terms of training costs, we employ standard backpropagation with at most two light-weight additional regularizers, which is not so different from classical approaches.
\section{Experimental Details}
\label{sec:experiments}
We evaluate \method{} on two sets of tasks, both requiring the ability to capture long-range interactions.\footnote{https://github.com/nec-research/Adaptive-Message-Passing} Additional node classifcation results can be found in Appendix \ref{sec:node-classification}.

\paragraph{Synthetic Datasets}
We consider the tasks of predicting the diameter, the single-source shortest paths (SSSP), and the node eccentricity on synthetic graphs \citep{corso_principal_2020}. In particular, we closely follow the setup of \citet{gravina_antisymmetric_2023} with graph sizes ranging from 25 to 35 nodes, topologies sampled from different graph generators, and each node has one random (sampled from a Normal distribution) feature attached. For SSSP, a binary feature is added to each node to indicate whether it is the source node in the graph or not. Each dataset amounts to 7040 graphs split into 5120 for training, 640 for validation, and 1280 for testing. The metric to be optimized is the $\log_{10}$ of the mean squared error (MSE). We apply early stopping on the validation MSE. \\
We have observed that the performance reported in \citet{gravina_antisymmetric_2023} can be improved by a significant margin if we average results over 20 rather than four final (that is, after model selection) training runs and increase the patience of the early stopper from 100 to 300, giving models more time to converge to a good solution. Therefore, to ensure a more robust set of results, we re-evaluated all baselines\footnote{We received support from the authors of paper.} considering these changes, and in many cases, we improved the scores. We combine \method{} with three message passing architectures, namely the Graph Convolutional Network (GCN) \citep{kipf_semisupervised_2017}, the Graph Isomorphism Network (GIN) \citep{xu_how_2019}, and the Anti-Symmetric DGN (ADGN) \citep{gravina_antisymmetric_2023}, and in addition we compare against GAT \citep{velickovic_graph_2018}, GraphSAGE \citep{hamilton_inductive_2017}, GCNII \citep{chen2020simple}, DGC \citep{wang2021dissecting}, and GRAND \citep{chamberlain2021grand}. Hyper-parameter details can be found in Appendix \ref{sec:hyper-params}.

\paragraph{Chemical Datasets}
We also test \method{} on real-world chemical graph prediction benchmarks, taken from the Long Range Graph Benchmark, called \textit{peptides-func} and \textit{peptides-struct} \citep{dwivedi_long_2022}. The first is an imbalanced multi-label graph classification dataset with ten total peptide functions, and we measure performances using the average precision (AP). The second is a multi-label graph regression task where we want to predict the peptides' properties based on their 3D information, and one evaluates the mean absolute error (MAE). Both datasets contain 15535 peptides with approximately $150$ nodes each, and the data is split into 70~\% for training, 15~\% for validation, and 15~\% for testing. We apply early stopping on the validation MAE. \\
We rely on the fair re-evaluation of \citet{tonshoff_did_2023} that shows how simple baselines like a GCN can achieve very competitive performances when properly tuned. In addition, we follow previous works \citep{rampavsek_recipe_2022,tonshoff_did_2023} and add random-walk structural encodings for \textit{peptides-func} and Laplacian positional encodings for \textit{peptides-struct}. For completeness, we include results from \citet{dwivedi_long_2022}, its re-evaluation \citep{tonshoff_did_2023}, and other results such as CRaWL \citep{toenshoff_graph_2021}, DRew \citep{gutteridge_drew_2023}, Exphormer \citep{shirzad_exphormer_2023}, GRIT \citep{ma_grit_2023}, Graph ViT and G-MLPMixer \citep{he_generalization_2023}, LASER \citep{barbero_locality_2023}, CO-GNN \citep{finkelshtein_cooperative_2023}, NBA \citep{park_non_2023}, PH-DGN \cite{heilig2025porthamiltonian}, GRED \citep{ding_recurrent_2023}, PR-MPNN \citep{qian_probabilistically_2023} and IPR-MPNN \citep{qian_probabilistic_2024}.

We evaluate \method{} on GCN, GINE \citep{hu_strategies_2020}, and GatedGCN \citep{bresson_residual_2017}; our grid search follows the best hyper-parameter reported by \citet{tonshoff_did_2023} (except the depth). As above, details on \method{}'s hyper-parameters can be found in Appendix \ref{sec:hyper-params}. Because the optimal depth ultimately depends on the task and the specific configuration of the model, we cannot impose arbitrary restrictions on the number of total parameters as done in \citet{dwivedi_long_2022}; instead, we are interested in letting the model freely adapt and choose the best parametrization that maximizes the performance.

\section{Results}
\label{sec:results}

\subsection{Quantitative Results}
\label{subsec:quantitative-results}
\begin{table*}[t]
\centering
\scriptsize
\begin{tabular}{lcccccc}
\toprule
             & Diameter & Rel Imp & SSSP & Rel Imp & Eccentricity & Rel Imp \\ \midrule
\textsc{GCN}          & $ 0.6146\pm{0.0375}$  & & $0.9132 \pm{0.0051}$   & & $0.7398\pm{0.0705}$  &\\
\textsc{GAT}          & $ 1.4367\pm{0.3558}$  & & $0.6070 \pm{0.0375}$   & & $1.0714\pm{0.0616}$  &\\
\textsc{GraphSAGE}    & $ 0.6146\pm{0.0744}$  & & $-1.0139\pm{0.0120}$   & & $ 1.0859\pm{0.0001}$ & \\
\textsc{GIN}          & $ 0.2408\pm{0.0154}$  & & $-0.2648\pm{0.4437}$   & & $ 0.9229\pm{0.0002}$ & \\ 
\textsc{GCNII}        & $ 0.5057\pm{0.0309}$  & & $-0.9172\pm{0.4396}$   & & $ 0.7112\pm{0.0255}$ & \\
\textsc{DGC}          & $0.5601\pm{0.0220}$   & & $-0.0254\pm{0.0077}$   & & $0.8051\pm{0.0017}$  &\\
\textsc{GRAND}        & $0.9477\pm{0.2160}$   & & $0.1909\pm{0.3103}$    & & $0.7450\pm{0.1369}$ & \\
\textsc{ADGN}         & $-0.4530 \pm{0.0883}$ & & $-3.5448\pm{0.2749}$   & & 0.0547$ \pm{0.0732}$ & \\ \hline
\method{}$_{\textsc{GCN}}$  & $-0.1072^{\dagger} \pm{0.0791}$ & \cellcolor{rwth-green!90}-81\% &  $0.5440^{\dagger} \pm{0.0108}$ & \cellcolor{rwth-green!90}-57\% & $0.6054^{\dagger} \pm{0.0919}$ & \cellcolor{rwth-green!60}-26\%\\
\method{}$_{\textsc{GIN}}$  & $-0.4874^{\dagger} \pm{0.1111}$ & \cellcolor{rwth-green!90}-81\% &  $-3.0628^{\dagger}\pm{0.3159}$ & \cellcolor{rwth-green!90}-99\% & $0.4093^{\dagger} \pm{0.0546}$ & \cellcolor{rwth-green!90}-69\% \\
\method{}$_{\textsc{ADGN}}$ &  $-0.5891^{\dagger} \pm{0.0720}$ & \cellcolor{rwth-green!60}-27\% & $-3.9579^{\dagger} \pm{0.0769}$ & \cellcolor{rwth-green!90}-61\% &  $0.0515^{\dagger} \pm{0.1819}$ & \cellcolor{rwth-green!20}-1\% \\ \hline 
Avg Rel Imp & & \cellcolor{rwth-green!90}-63\% & & \cellcolor{rwth-green!90}-72\% & & \cellcolor{rwth-green!60}-32\%\\ \bottomrule
\end{tabular}
\caption{Mean $log_{10}(\mathrm{MSE})$ and standard deviation averaged over 20 final runs on Diameter, SSSP, and Eccentricity. A $\dagger$ indicates that \method{} yields an improvement in the mean score compared to the base model.}
\label{tab:Syntethic-datasets-300epochs}
\end{table*}
\begin{table}[t]
\centering
\scriptsize
\begin{tabular}{lcc}
\toprule
Method     & peptides-func & peptides-struct \\
           & Test AP $\uparrow$ & Test MAE $\downarrow$ \\
\midrule
\multirow{6}*{\rotatebox{90}{\textsc{LRGB}}}
\ \ \textsc{GCN}         & $0.5930 \pm 0.0023$ & $0.3496 \pm 0.0013$ \\
\ \ \ \ \ \ \textsc{GINE}        & $0.5498 \pm 0.0079$ & $0.3547 \pm 0.0045$ \\
\ \ \ \ \ \ \textsc{GatedGCN}    & $0.6069 \pm 0.0035$ & $0.3357 \pm 0.0006$ \\
\ \ \ \ \ \ \textsc{Transformer} & $0.6326 \pm 0.0126$ & $0.2529 \pm 0.0016$ \\
\ \ \ \ \ \ \textsc{SAN}         & $0.6439 \pm 0.0075$ & $0.2545 \pm 0.0012$ \\
\ \ \ \ \ \ \textsc{GPS}         & $0.6535 \pm 0.0041$ & $0.2500 \pm 0.0005$ \\
\midrule
\multirow{4}*{\rotatebox{90}{\textsc{Re-Eval}}}
\ \ \textsc{GCN}         & $0.6860 \pm 0.0050$ & $0.2460 \pm 0.0007$ \\
\ \ \ \ \ \ \textsc{GINE}        & $0.6621 \pm 0.0067$ & $0.2473 \pm 0.0017$ \\
\ \ \ \ \ \ \textsc{GatedGCN}    & $0.6765 \pm 0.0047$ & $0.2477 \pm 0.0009$ \\
\ \ \ \ \ \ \textsc{GPS}         & $0.6534 \pm 0.0091$ & $0.2509 \pm 0.0014$ \\
\midrule
\multirow{13}*{\rotatebox{90}{\textsc{Others}}}
\ \ {\textsc{CRaWl}}  & $0.7074 \pm 0.0032$ & $0.2506 \pm 0.0022$ \\
\ \ \ \ \ \ \textsc{DRew}$_{\textsc{GCN}}$           & $0.7150 \pm 0.0044$ & $0.2536 \pm 0.0015$ \\
\ \ \ \ \ \ \textsc{DRew}$_{\textsc{GatedGCN}}$     & $0.6977 \pm 0.0026$ & $0.2539 \pm 0.0007$ \\
\ \ \ \ \ \ \textsc{Exphormer}         & $0.6527 \pm 0.0043$ & $0.2481 \pm 0.0007$ \\
\ \ \ \ \ \ \textsc{GRIT}              & $0.6988 \pm 0.0082$ & $0.2460 \pm 0.0012$ \\
\ \ \ \ \ \ \textsc{Graph ViT}         & $0.6942 \pm 0.0075$ & $0.2449 \pm 0.0016$ \\
\ \ \ \ \ \ \textsc{G-MLPMixer}        & $0.6921 \pm 0.0054$ & $0.2475 \pm 0.0015$ \\
\ \ \ \ \ \ \textsc{LASER}             & $0.6440 \pm 0.0010$ & $0.3043 \pm 0.0019$ \\
\ \ \ \ \ \ \textsc{CO-GNN}            & $0.6990 \pm 0.0093$ & - \\
\ \ \ \ \ \ \textsc{NBA}$_{\textsc{GCN}}$           & $0.7207 \pm 0.0028$ & $0.2472 \pm 0.0008$  \\
\ \ \ \ \ \ \textsc{NBA}$_{\textsc{GatedGCN}}$           & $0.6982 \pm 0.0014$ & $0.2466 \pm 0.0012$  \\
\ \ \ \ \ \ \textsc{PH-DGN}           & $0.7012 \pm 0.0045$ & $0.2465 \pm 0.0020$ \\
\ \ \ \ \ \ \textsc{GRED}           & $0.7041 \pm 0.0049$ & $0.2584 \pm 0.0015$ \\
\ \ \ \ \ \ \textsc{PR-MPNN}           & $0.6825 \pm 0.0086$ & $0.2477 \pm 0.0005$ \\
\ \ \ \ \ \ \textsc{IPR-MPNN}           & $0.7210 \pm 0.0039$ & $0.2422 \pm 0.0007$ \\
\midrule
\method{}$_{\textsc{GCN}}$ & $0.7161^\dagger \pm 0.0047$ & $0.2446^\dagger \pm 0.0026$ \\
\method{}$_{\textsc{GINE}}$ & $0.7065^\dagger \pm 0.0105 $ & $0.2468^\dagger \pm 0.0026$ \\
\method{}$_{\textsc{GatedGCN}}$ &  $0.6943^\dagger \pm 0.0046 $ & $0.2480^\dagger \pm 0.0012$\\
\bottomrule
\end{tabular}
\caption{Mean test scores and standard deviation averaged over 4 final runs on the chemical datasets. The $\dagger$ indicates that \method{} yields an improvement compared to the base architecture.}
\label{tab:results-chemical-datasets}
\end{table}
Table \ref{tab:Syntethic-datasets-300epochs} reports the test scores for the Diameter, SSSP, and Eccentricity datasets for all baselines and \method{} versions. On all datasets, \method{} always grants a reduction of the test error, with an average improvement of 63~\% on Diameter, 72~\% on SSSP, and 32~\% on Eccentricity. These results show that learning the proper depth of a network and a policy for filtering messages exchanged between nodes is more effective than relying on a manually crafted grid search and a fully synchronous message passing behavior. Eccentricity is the most difficult task to solve, whereas one could claim that SSSP is almost solved for the graphs considered. We achieved the greatest reduction in error with respect to the GIN model, probably because the authors in \citet{gravina_antisymmetric_2023} found that a 1-layer GIN was the best configuration across all tasks after tuning the depth. This stresses the positive impact of letting the model learn how and when to propagate messages.
On the chemical datasets (Table \ref{tab:results-chemical-datasets}), we observe a similar trend. Regardless of the base message passing architecture, \method{} consistently improves its performance on classification and regression tasks. On \textit{peptides-func}, we achieve an improvement of 2 to 3.4 \% compared to the base models and a reduction of MAE on \textit{peptides-struct} that positions all \method{} versions at state-of-the-art levels (considering overlap of standard deviations). Our analyses also found that the number of hidden units is an important hyper-parameter to perform well on these tasks, and a larger value seems to correlate well with good performances. Combined with the above results, we argue that the parameter budget imposed by previous works \citep{dwivedi_long_2022} might limit the future progress on these tasks, as deeper networks are probably needed (we provide an analysis of the depth found by \method{} below). The average diameter of these peptides is $57$, meaning that using ten layers as done in other works might not be enough to capture long-range dependencies \citep{tonshoff_did_2023}. We provide a visual analysis of predictions in Appendix \ref{sec:results-predictions}.

\subsection{\method{} mitigates oversmoothing and oversquashing}
\label{subsec:results-oversmoothing-oversquashing}
\begin{figure}[th]
    \centering
    \includegraphics[width=\columnwidth]{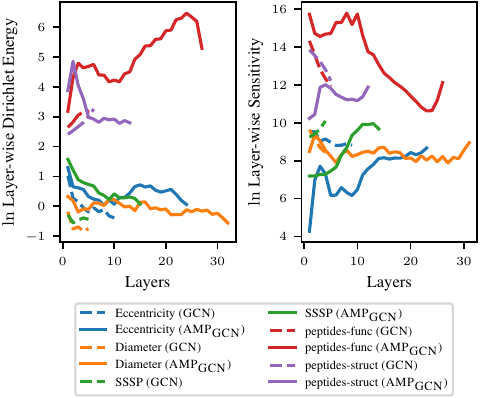}
    \caption{We show the Dirichlet energy (left) and the sensitivity (right) across layers for the GCN model and its \method{} version.}
    \label{fig:oversmoothing-oversquashing-analysis}
\end{figure}
\begin{figure*}[t]
    \centering
    \begin{subfigure}
        {\includegraphics{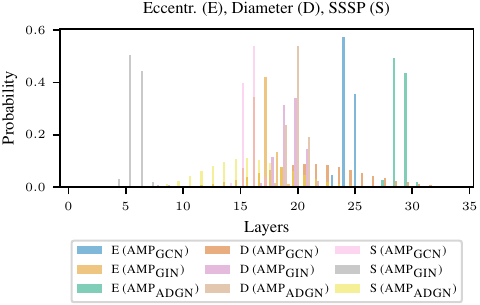}}
    \end{subfigure}
    \begin{subfigure}
        {\includegraphics{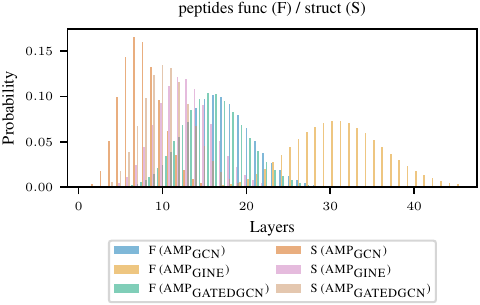}}
    \end{subfigure}
    \caption{We show the distribution learned by the best configurations of each base model on the synthetic and chemical datasets.}    
    \label{fig:depth-analysis}
\end{figure*}

We now comment on \method{}'s ability to mitigate oversmoothing and oversquashing, and we refer to Figure \ref{fig:oversmoothing-oversquashing-analysis} for a qualitative analysis of the former (left) and of the latter (right). First, we computed the logarithm of the Dirichlet energy over embeddings of a trained GCN for different layers and datasets. 
This analysis reveals that the Dirichlet energy for \method{}'s variants is typically higher than the corresponding baselines and it can exhibit a stable, decreasing, or increasing behavior as the depth grows, in contrast to existing theoretical research on the GCN model where the Dirichlet energy constantly decreases and embeddings converge to the same value \citep{li_deeper_2018,rusch_survey_2023}; note that we apply skip connections to the base GCN, so the energy does not immediately decreases. Therefore, it appears that our approach is indeed capable of \textit{controlling} oversmoothing; we attribute this to the combination of message filtering and a layer-wise loss, which favors the propagation of gradient to intermediate layers. We also report the layer-wise logarithm of node embeddings' sensitivity\footnote{Averaged over 250 validation nodes due to prohibitive costs.} as the gradient of the embeddings of the last layer $L$ with respect to the ones of intermediate layers~$l$: $\sum_{(v,u) \in \mathcal{V}}\left|\left| \frac{\partial \bm{h}_v^L}{\partial \bm{h}^{\ell}_u} \right|\right|_{1}$. Sensitivity provides insights into how pruning messages affects oversquashing; in fact, filtering messages might reduce said sensitivity with respect to the input. We report a quite heterogeneous picture in Figure \ref{fig:oversmoothing-oversquashing-analysis} (right): the sensitivity of \method{} can peak at the first or last layers, increase abruptly, or remain relatively stable. In all these cases, we have seen how \method{}$_\textsc{GCN}$ substantially improves performances on tasks where it seems important to address oversquashing. This evidence also warns us against using sensitivity as the sole metric to measure oversquashing as a performance bottleneck.

\subsection{Analysis of the \method{}'s Learned Depth}
\label{subsec:results-layering}
\begin{figure*}[!t]
    \centering
    \begin{subfigure}
        {\includegraphics{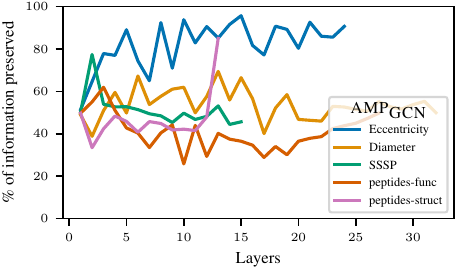}}
    \end{subfigure}
    \begin{subfigure}
        {\includegraphics{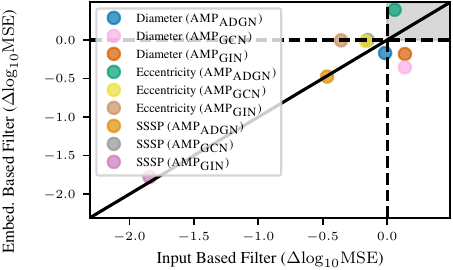}}
    \end{subfigure}
    \begin{subfigure}
        {\includegraphics{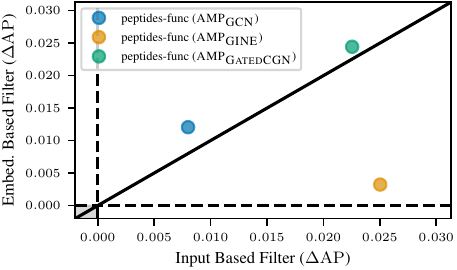}}
    \end{subfigure}
    \begin{subfigure}
        {\includegraphics{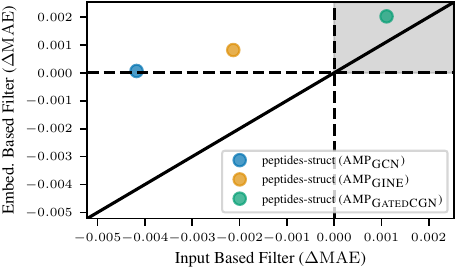}}
    \end{subfigure}
    \caption{(Top right) we visualize the amount of information preserved in each layer by \method{}$_{\textsc{GCN}}$. (Others) Ablation study of message filtering scheme: If a point lies in the area represented by the gray color, then filtering is not beneficial.}
    \label{fig:ablation-filtering}
\end{figure*}
To understand how \method{} mitigates underreaching, we inspect the learned distributions for synthetic datasets in Figure~\ref{fig:depth-analysis}~(left). \method{} uses more layers than the baselines (details in Appendix \ref{sec:base-models-depth}) to achieve the best score on Eccentricity. Instead, it is found that about 20 layers are necessary to solve the task for the Diameter dataset, with all runs  attaining a mean value between 17 and 22 layers. Finally, SSSP seems to be the task that requires fewer layers on average, with \method{}$_{\textsc{GIN}}$ selecting less than ten layers to reach a very competitive score. Overall, it appears that folded normal and mixtures of folded normal distributions were selected more frequently as the best hyper-parameter for the synthetic tasks; in particular, the distributions for Eccentricity look sharply peaked, as if the models would need to use only the information computed at the very end of the deep architecture. It is worth remarking that this behavior is completely adaptive and guided by the task.
Finally, we observe that the distributions learned on the chemical datasets are mostly Poisson ones, and \method{} learns deeper networks than the corresponding baselines to reach better scores. These distributions peak at around ten layers for \textit{peptides-struct}, which is more or less in line with what was reported in previous works \citep{tonshoff_did_2023}. In any case, \method{} enables training of very deep architectures thanks to its layer-wise output. However, this has a non-negligible cost in terms of number of readout parameters.

\subsection{On the Effects of Message Filtering}
\label{subsec:results-filtering}
Figure \ref{fig:ablation-filtering} visualizes the amount of information pruned at each layer by \method{}$_{\textsc{GCN}}$ on all datasets, together with an ablation study about the benefits of message filtering. The amount of information pruned is computed by summing the message filters' activations and normalizing the result by the total number of messages exchanged at each layer. We can see how \method{}$_{\textsc{GCN}}$ gradually increases the amount of information to be used for Eccentricity, whereas in \textit{peptides-func}, this quantity is almost always below 50\%. One can appreciate how, depending on the task, the behavior of the message filtering changes significantly.

The ablation study, on the other hand, provides evidence that message filtering is, in most cases, a good strategy for performance improvements. We see that, for each model and dataset, input-based and embedding-based filtering provide a positive improvement in scores compared to no-filtering. Points lying in the grey area correspond to no improvement. We conclude that the choice of which filtering strategy to use remains a matter of empirical investigations.
\section{Conclusions}
\label{sec:conclusions}
We have introduced Adaptive Message Passing, a probabilistic framework that endows message passing architectures with the ability to learn how many messages to exchange between nodes and which messages to filter out. Our approach actively targets the long-range issue by relying on the observation that filtering messages mitigates oversmoothing and oversquashing, whereas learning depth can ideally solve underreaching. 
\method{} achieves competitive results on long-range datasets without imposing strong inductive biases.
Through qualitative analyses, our findings reveal how \method{} learns very deep architectures if necessary for the task, and the amount of information propagated can greatly be reduced compared to classical message passing.
Overall, our approach suggests that it might not be necessary to alter the initial graph structure, e.g., through rewiring, to improve the performances on long-range tasks.
We hope Adaptive message passing will foster exciting research opportunities in the graph machine learning field and find successful applications in physics, chemistry, and material sciences.

\section*{Impact Statement}
This paper presents work whose goal is to advance the field of Machine Learning. There are many potential societal consequences of our work, none of which we feel must be specifically highlighted here.


\bibliography{bibliography}
\bibliographystyle{icml2025}

\newpage
\appendix
\onecolumn
\section{Extended Related Work Section}
\label{sec:extended-related work}
Below we provide a more detailed related work section that does not fit in the main paper.

\paragraph{Oversquashing.}
There are many methods that attempt to address the oversquashing problem with the goal of better capturing long-range dependencies \citep{alon_on_2021,li2024longshortrange}. There is agreement that modifying the message passing scheme leads to improved performances; in this sense, the graph structure does not match exactly the computational graph used to compute the node embeddings. Some works learn how a node should completely stop propagating a message in a fixed-depth architecture \citep{spinelli_adaptive_2020} or if it should only listen, isolate, or receive/broadcast its own message \citep{finkelshtein_cooperative_2023}. Similarly, one can learn to sample edges at each message passing layer according to some learned parametrization \citep{hasanzadeh_bayesian_2020} or have a completely asynchronous message passing \citep{faber_asynchronous_2023}. Our work differs from these works as we apply a \textit{learned} (soft) filtering to all existing messages.

Another idea is to modify message passing to avoid backtracking of messages back to the source node, to achieve less redundancy of information \citep{park_non_2023}. While this choice proves effective at several tasks, it is still an open question whether it is always the best choice for the task at hand. In attention-based approaches \citep{velickovic_graph_2018}, an edge filter is computed using some non-linear relationship between the embeddings of the source and destination nodes. This can introduce a severe computational burden as the function needs to be applied to all edges.
Similarly, GNN-FiLM \citep{brockschmidt_gnn_2020} learns a feature-wise linear modulator that depends on the destination node and modulates the magnitude of all incoming messages.

On the other hand, rewiring approaches try to alter the graph connectivity rather than the message passing operation. This action is meant to increase the sensitivity \citep{topping_understanding_2022} of a node with respect to another, and it has been theoretically linked to the oversquashing problem. Some recent works try to preserve locality and sparsity of the rewiring process \citep{barbero_locality_2023} or dynamically rewire the graph based on the layers \citep{gutteridge_drew_2023}. In contrast, others take a probabilistic approach to rewiring based on sampled sub-graphs \citep{qian_probabilistically_2023}. Recently, a critical perspective on the effectiveness of rewiring approaches has also been given \citep{tortorella_leave_2022}. \\
Finally, we mention ordinary differential equation-based message passing approaches, which are provably preserving information regardless of the depth in the network \citep{gravina_antisymmetric_2023} and have shown great results on datasets aimed at capturing long-range dependencies.

\paragraph{Oversmoothing.}
Oversmoothing is perhaps one of the first problems that emerged empirically and was then analyzed theoretically \citep{li_deeper_2018,oono_graph_2020,rusch_survey_2023}. Not surprisingly, one practical solution to oversmoothing is dropping edges to reduce the overall flow of messages and, thus, avoid the convergence of all embeddings to the same value \citep{rong_dropedge_2020}. Another well-known solution to alleviate oversmoothing is to employ skip/residual connections \citep{kipf_semisupervised_2017,li_deepgcns_2019}, which consists of summing the representations learned at deeper layers with those of previous ones. Similarly to what is done in this work, the concatenation of node representations across layers is also a way to circumvent oversmoothing, which has been adopted in neural and probabilistic models to improve the downstream performances on several node and graph-related tasks \citep{bacciu_contextual_2018,xu_representation_2018,bacciu_probabilistic_2020}.
Instead, an orthogonal research direction considers implicit neural networks for graphs that correspond to infinite-depth models and seem to be able to capture long-range dependencies \citep{poli_continuous_2021,liu_eignn_2021}. These models simulate synchronous message passing with a potentially infinite number of message-propagation steps, and some of them appear to be empirically robust to the oversmoothing problem.

\paragraph{Adaptive Architectures.}
The last part of this section is dedicated to works that try to learn the architecture of the model during training. Our work is inspired by the unbounded depth networks (UDNs) of \citet{nazaret_variational_2022}, who proposed a variational framework for learning the depth in deep neural networks. In the graph domain, the first approach to learning the depth of a DGN was proposed by \citet{micheli_neural_2009}, who applied the cascade correlation algorithm \citep{fahlman_cascade_1989} to learn a proper depth for the task. In the field of graph representation learning, other works attempted to learn the width of the representation of each message passing layer by exploiting Bayesian non-parametric models \citep{castellana_infinite_2022}, which allows to save time and memory when building deeper probabilistic DGNs.
Finally, it is important to notice that these works, including this manuscript, are all orthogonal to the popular field of neural architecture search \citep{zhou_auto_2022}: The former attempts at dynamically modifying the architecture \textit{during learning}, whereas neural architecture search approaches find smarter ways to carry out a grid search. An advantage of adaptive approaches is that they can greatly reduce time and computational costs to perform a hyper-parameter search.

\paragraph{Graph Rewiring, Structure Learning, and Transformer Models for Graphs.}

There has been a growing interest in adapting the input graph structure and, therefore, the message-passing operations of DGNs. These approaches can be roughly divided into graph structure learning and rewiring methods on one hand and transformer-based models on the other. 
There are several strategies of graph rewiring such as incorporating multi-hop neighbors \citep{bruel_rewiring_2023} and nodes reachable through shortest paths \citep{abboud_shortest_2022}. \citet{gutteridge_drew_2023} is most related to the filtering approach we introduce as it rewires the graph for every message-passing layer.  The methods MixHop \citep{haija_mixhop_2019}, SIGN \citep{frasca_sign_2020}, and DIGL \citep{gasteiger_diffusion_2019} can also be considered as graph rewiring as these leverage different heuristics to reach further-away neighbors in a single message-passing layer. \citet{deac_expander_2022,shirzad_exphormer_2023} use the notion of  expander graphs to alter the messages, while \citet{karhadkar_fosr_2023} resort to spectral techniques, and \citet{banerjee_oversquashing_2022} propose a random edge flip approach. There also exist rewiring heuristics based on particular metrics such as Ricci and Forman curvature~\citep{bober_rewring_2022}.  Finally, recent work has proposed an approach to rewire the input graph probabilistically and in a data-driven way \citep{qian_probabilistically_2023}. 

Graph structure learning (GSL) approaches are another line of work rewiring the input graphs. Here the focus is on node classification problems on a single graph. Typically, these methods maintain a learnable function that assigns prior scores to the edges, and based on these scores a subset of edges is selected from the original graph~\citep{chen_iterative_2020,yu_graph_2020,zhao_data_2021}. To introduce discreteness and sparsity, \citet{franceschi_learning_2019,kazi_differentiable_2022,zhao_data_2021} use discrete categorical (sampled from using the Gumbel softmax trick) and Bernoulli  distributions, respectively. Most existing GSL approaches typically use a $k$-NN algorithm, a simple randomized version of $k$-NN, or represent edges with independent Bernoulli random variables. For a comprehensive survey of GSL, see~\citet{fatemi_ugsl_2023,zhou_opengsl_2023}. In the context of node classification, there has been recent progress in a more principled understanding of the possible advantages of GSL in the fully supervised setting~\citep{castellana_investigating_2023, errica_class_2023}.  

Graph transformers~\citep{dwivedi_long_2022,he_generalization_2023,muller_attending_2023,rampavsek_recipe_2022, chen_structure_2022} adaptively change the message-passing operations for each layer by applying an attention mechanism among all nodes of the input graphs. Experimental results have shown that graph transformers have the ability to mitigate over-squashing~\citep{muller_attending_2023}. Due to their attention mechanism, however, transformer-based models have typically a quadratic space and memory requirement. 

State space models have also been applied to graphs \citep{wang2024graph}. Similarly, theoretically grounded sequence-processing frameworks \citep{toth2022capturing,gruber2024processing}, leveraging randomized signatures, demonstrated promising potential in alleviating oversquashing effects in large graphs. The main differences with AMP, \citet{wang2024graph} relies on a separate sequence model to develop a node selection mechanism whereas we work on the message passing itself; \citet{toth2022capturing} develops a new graph Laplacian that is better suited for long-range propagation; \citet{gruber2024processing} converts a graph into a latent representation that can be passed to downstream classifiers.

Contrary to all these approaches, AMP, the adaptive message passing approach presented here, uses a variational framework to jointly learn both the depth of the DGN and a filtering of the messages passed between nodes in each of these layers.  

\section{Generalization to New Families of Truncated Distributions}
\label{sec:extension-truncated}
The family of \quotes{truncated} Poisson distributions, proposed by \citet{nazaret_variational_2022} to learn unbounded depth networks, satisfies specific requirements that allow us to efficiently perform (variational) inference. In particular, by truncating the Poisson distribution at its quantile function evaluated at $c$, one can bound its support and compute expectations in finite time. However, the Poisson distribution suffers from equidispersion, meaning that the variance is equal to the mean; this is a particularly limiting scenario when learning distributions over the importance of layers. In fact, one might also want to model variances that are smaller or greater than the mean, which is referred to as under and over-dispersion, respectively, to learn a broader class of distributions \citep{castillo_overdispersed_2005,harris_modeling_2012}. To address this problem, in the following we introduce two families of distributions and prove that they also satisfy the requirements defined in \citet{nazaret_variational_2022}; we formally recall such requirements below.

\begin{definition}
    A variational family $Q = {q(\bm{\omega})}$ over $\mathbb{N}^+$ is \textit{unbounded} with \textit{connected} and \textit{bounded} members if
    \begin{enumerate}
        \item $\forall q \in Q$, $\operatorname{support}(q)$ is bounded
        \item $\forall L \in \mathbb{N}^+, \exists q \in Q$ such that $L \in \operatorname{argmax}(q)$
        \item Each parameter in the set $\bm{\omega}$ is a continuous variable.        
    \end{enumerate}
    \label{def:truncated-family}
\end{definition}

Condition 1 is necessary to compute the expectation over $q(\ell ; \bm{\lambda})$ in finite time, condition 2 ensures that we can give enough probability mass to each point in the support of $q$, and condition 3 is required for learning the distributions' parameters in a differentiable manner.

\paragraph{The Discrete Folded Normal Distribution.}
folded normal (FN) distributions \citep{leone_folded_1961} can model under-, equi-, and over-dispersion. They are parametrized by a mean parameter $\mu$ and a standard deviation $\sigma$. Its density is defined as
\begin{align}
  p_{\text{FN}}(x;\mu,\sigma) = \frac{1}{\sqrt{2\pi\sigma^2}} e^{-\frac{(x - \mu)^2}{2\sigma^2}} + \frac{1}{\sqrt{2\pi\sigma^2}} e^{-\frac{(x + \mu)^2}{2\sigma^2}}, \ \mu,\sigma \in \mathbb{R}, x \ge 0.
\end{align}
To use a FN distribution in \method{}, the idea is to first define a discrete version of the folded normal (DFN) with the strategy highlighted in \citet{roy_discrete_2003}:
\begin{align}
    & p_{\text{DFN}}(0;\mu,\sigma) = S_{\text{FN}}(1; \mu,\sigma) \\ 
    & p_{\text{DFN}}(x;\mu,\sigma) = S_{\text{FN}}(x+1; \mu,\sigma) - S_{\text{FN}}(x; \mu,\sigma), \forall x \in \mathbb{N}^+ 
\end{align}
where $S_{\text{FN}}(x; \mu,\sigma)$ is the cumulative distribution function (c.d.f.) of the folded normal distribution evaluated at $x$.\footnote{We note that the support is defined over $\mathbb{N}$ and not on $\mathbb{N}^+$, but this is not an issue from a practical point of view.} It is also useful to notice the equivalence between the c.d.f of the DFN $S_{\text{DFN}}(x;\mu,\sigma)$ and that of a folded normal $S_{\text{FN}}(x;\mu,\sigma)$
\begin{align}
 & S_{\text{DFN}}(x;\mu,\sigma) = \sum_{i=0}^x p_{\text{DFN}}(x;\mu,\sigma) =  S_{\text{FN}}(x+1; \mu,\sigma), x \in \mathbb{N},
 \label{eq:fn-dfn-equivalence}
\end{align}
which implies that $S_{\text{DFN}}(x;\mu,\sigma) \ge S_{\text{FN}}(x;\mu,\sigma)$. Figure \ref{fig:truncated-bounds} shows the probability mass function (p.m.f.) of a DFN distribution with $\mu=1$ and $\sigma=5$ and its cumulative mass function (c.m.f.).

\begin{figure}[t]
    \centering
    \begin{subfigure}
        {\includegraphics{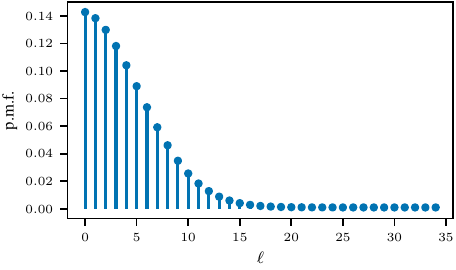}}
    \end{subfigure}
    \begin{subfigure}
        {\includegraphics{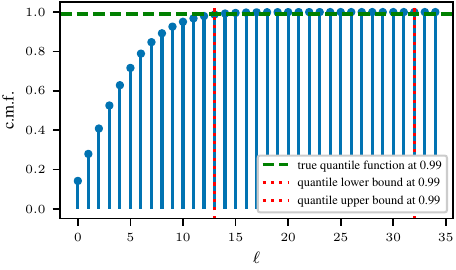}}
    \end{subfigure}
    \caption{(Left) Probability mass function of a DFN distribution with $\mu=1$ and $\sigma=5$. (Right) The cumulative mass function of the distribution with an example of lower and upper bounds (vertical dashed lines) for the true quantile function evaluated at $0.99$ of the corresponding FN distribution (horizontal dashed line).}
    \label{fig:truncated-bounds}
\end{figure}

Clearly, condition 3 of Definition \ref{def:truncated-family} is satisfied. It is also trivial to satisfy condition 2 by choosing a peaked distribution with a small value of $\sigma$. In what follows, we focus on lower and upper bounds to the quantile function evaluated at $c$ of the DFN distribution so that we know we can truncate the distribution to the \textit{finite} quantile threshold, meaning condition 1 is also satisfied.

\begin{theorem}
    There exists lower and upper bounds to the quantile evaluated at $c$, $0 < c < 1$, for any DFN distribution with $\sigma > 0$.
\end{theorem}
\begin{proof}
    We first need to compute a lower bound to the quantile function of the FN distribution since there is no closed formula for it. To start, we note that the c.d.f. of the Gaussian distribution is greater or equal to that of a folded normal distribution:
    \begin{align}
       \frac{1}{2} \text{erf}\left(\frac{x-\mu}{\sigma\sqrt{2}}\right) + \frac{1}{2} \ge \frac{1}{2} \text{erf}\left(\frac{x-\mu}{\sigma\sqrt{2}}\right) + \underbrace{\frac{1}{2} \text{erf}\left(\frac{x+\mu}{\sigma\sqrt{2}}\right)}_{\le \frac{1}{2}} = S_{\text{FN}}(x; \mu,\sigma).
    \end{align}
    where \quotes{erf} is the error function. This implies that the quantile threshold of the Gaussian $x_G$, which we know how to compute, is reached earlier than that of the FN $x_{\text{FN}}$, that is, $x_G \le x_{\text{FN}}$, and in particular $\left\lfloor x_G \right\rfloor \le \left\lfloor x_{\text{FN}} \right\rfloor$ are also lower bounds. It then follows from Equation \ref{eq:fn-dfn-equivalence} that $\left\lfloor x_G \right\rfloor$ - 1 is a lower bound for the DFN distribution.

To find an upper bound, we apply Chernoff's Bound
\begin{align}
    p(X \ge x) \le \frac{M_X(t)}{e^{tx}}, \ \  \forall t > 0
\end{align}
where $X$ is a \rv{} that follows a folded normal distribution with mean $\mu$ and standard deviation $\sigma$, and $M_X(t)=\mathbb{E}[e^{tX}]$ is the well-known moment generating function of the $X$. To find an upper bound to the quantile threshold, we need that $(1-c) = \frac{M_X(t)}{e^{tx}}$ for some choice of $t$. Defining $\Phi$ as the normal \textit{c.d.f.} $\Phi(x)= \frac{1}{2}\left[1 + \text{erf}\left(\frac{x}{\sqrt{2}}\right)\right] \ge 0$, we choose $t=\frac{1}{\sigma}$ and obtain
\begin{align}
    \frac{M_X(t)}{e^{tx}} & = e^{-tx}\left(e^{\frac{\sigma^2t^2}{2}+\mu t}\Phi\left(\frac{\mu}{\sigma} + 1 \right) + e^{\frac{\sigma^2t^2}{2} - \mu t}\Phi\left(-\frac{\mu}{\sigma} + 1 \right)\right) \\
    & = e^{-\frac{x}{\sigma}}e^{\frac{1}{2}}e^{\frac{\mu}{\sigma}}\left(\Phi\left(\frac{\mu}{\sigma} + 1 \right) + \Phi\left(-\frac{\mu}{\sigma} + 1 \right) \frac{1}{e^{2\mu / \sigma}}\right) \\
    & = k e^{\frac{\mu-x}{\sigma}}, \\ & k = e^{\frac{1}{2}}\left(\Phi\left(\frac{\mu}{\sigma} + 1 \right) + \Phi\left(-\frac{\mu}{\sigma} + 1 \right) \frac{1}{e^{2\mu / \sigma}}\right) > 0.
\end{align}
Therefore, the upper bound of the quantile threshold is given by
\begin{align}
    & k e^{\frac{\mu-x}{\sigma}} = (1-c) \\
    & \ln k + \frac{\mu-x}{\sigma} = \ln (1-c) \\
    & \sigma \ln k + \mu-x = \sigma \ln (1-c) \\
    & x =  \mu + \sigma \ln k - \sigma \ln (1-c).
\end{align}

Therefore, if the upper bound to the quantile of the FN is $x$, it follows from Equation \ref{eq:fn-dfn-equivalence} that $x-1$ is also an upper bound of the DFN.
\end{proof}
Consequently, we can efficiently find the true quantile threshold by running a binary search between the lower and the upper bounds. Figure \ref{fig:truncated-bounds} (right) shows an example of lower and upper bounds (vertical dashed lines) as well as the true quantile threshold of the FN distribution.

\paragraph{A Mixture of Simpler Distributions.}
It is possible to learn more complex distributions $q(\ell ; \bm{\omega})$ that satisfy the conditions of Definition \ref{def:truncated-family} by mixing simpler distributions like the DFN defined above. A mixture of $C$ families of unbounded distributions $q_1(\ell ; \bm{\omega}), \dots, q_C(\ell ; \bm{\omega})$ with bounded and connected members is defined as:
\begin{align}
    q_{\mathcal{M}}(\ell ; \bm{\omega}) = \sum_{i=0}^C w_i q_i(\ell ; \bm{\omega})
\end{align}
where $0 \le w_i \le 1$ is mixture's $i$ weight and $\sum_{i=0}^C w_i = 1$. Conditions 2 and 3 are again trivially satisfied (a mixture can always collapse to one of its distributions that satisfy said conditions), and below, we show that lower and upper bounds still exist.

\begin{theorem}
    There exist lower and upper bounds to the quantile evaluated at $c$, $0 < c < 1$, for a mixture of $C$ distributions that satisfy the conditions of Definition \ref{def:truncated-family}, provided that lower and upper bounds exist for each distribution in the mixture.
\end{theorem}
\begin{proof}
The c.m.f. of a mixture of discrete distributions can be written as a weighted sum of c.m.f.s:
\begin{align}
    S_{\mathcal{M}}(x; \bm{\omega}) = \sum_{i=0}^C w_i S_i(x ; \bm{\omega}).
\end{align}
Let $x^*$ be the greatest upper bound of the quantile threshold across all $C$ components of the mixture, and let $i^*$ be the associated component. It follows that, $\forall j, S_j(x^* ; \bm{\omega}) \ge c$, and
\begin{align}
    S_{\mathcal{M}}(x^*; \bm{\omega}) = \sum_{i=0}^C w_i S_i(x^* ; \bm{\omega}) \ge \sum_{i=0}^C w_i c = c.
\end{align}
Therefore, $x^*$ is also an upper bound for the mixture of distributions. It is possible to prove that a lower bound of the mixture is the smallest lower bound of the quantile threshold across all $C$ components using a similar approach.
\end{proof}

To summarize, we have shown how one can use more complex families of distributions in the context of \method{}, allowing us to model under and over-dispersion. In this work, we will treat the choice of the family of distributions $q(\ell ; \bm{\omega})$ to use as a hyper-parameter to be tuned.

\section{ELBO Derivation}
\label{sec:elbo-derivation}

We now report the full derivation of our ELBO for a single sample:
\begin{align}
\ln p(g_i, Y_{g_i}) &\ge \mathbb{E}_{q(\bm{\theta}, L, \bm{F}_i, | g_i, Y_{g_i})}\left[\ln p(Y_{g_i}, L, \bm{F}_i, \bm{\theta} | g_i) - \ln q(L, \bm{F}_i, \bm{\theta} | g_i)\right] \\
&= \mathbb{E}_{q(L ; \bm{\lambda})}\Bigg[\ln\frac{p(L)}{q(L ; \bm{\lambda})} + \mathbb{E}_{q(\bm{\theta} | L ; \bm{\nu})}\left[\ln\frac{p(\bm{\theta})}{q(\bm{\theta} | L ; \bm{\nu})}\right] \\
& \hspace{5mm} + \mathbb{E}_{q(\bm{\theta} | L ; \bm{\nu})q(\bm{F}_i | g_i, L, \bm{\theta})}\Bigg[\ln\frac{p(\bm{F}_i)}{q(\bm{F}_i | g_i, L, \bm{\theta})} + \ln p(Y_{g_i} | L, \bm{F}_i, \bm{\theta} | g_i)\Bigg]\Bigg]\\
&= \sum_{\ell=1}^{\hat{L}}q(\ell ; \bm{\lambda})\Biggl[\ln\frac{p(\ell)}{q(\ell ; \bm{\lambda})} + \ln\frac{p(\bm{\nu})}{q(\bm{\nu} | \ell ; \bm{\nu})} + \ln\frac{p(\bm{F}_i)}{q(\bm{F}_i | g_i, \ell, \bm{\nu})} + \ln p(Y_{g_i} | \ell, \bm{F}_i, \bm{\nu}, g_i)\Biggr],
\end{align}

The extension to the full dataset is identical except for the decomposition of the rightmost term into a product of conditional probabilities, one for each \textit{i.i.d.}\ sample.

\section{On Oversmoothing, Oversquashing, and Underreaching}
\label{sec:discussion-oversquashing-underreaching-oversmoothing}
In this section, we discuss \method{}'s implications on \textit{oversmoothing}, \textit{oversquashing}, and \textit{underreaching}, all of which hamper the ability of DGNs to capture long-range interactions between nodes in the graph and are related in subtle ways.

\paragraph{Oversmoothing.} Oversmoothing has been formally defined by \citet{rusch_survey_2023} as the 
convergence of a node embeddings' similarity as the number of message passing layers increases. In other words, it formalizes the widely accepted notion that node embeddings tend to become identical after many layers of message passing. Different oversmoothing metrics have been proposed, and in this work, we consider the Dirichlet energy \citep{rusch_graph_2022,rusch_survey_2023} at layer $\ell$ defined as
\begin{align}
E(\bm{H^{\ell}}) = \frac{1}{|\mathcal{V}|} \sum_{u \in \mathcal{V}}\sum_{v \in \mathcal{N}_u} || \bm{h}_u^{\ell} - \bm{h}_v^{\ell} ||^2
\end{align}
where we indicate with $\bm{H^{\ell}}$ the set of node embeddings computed at layer $\ell$. 

There are at least two reasons why \method{} alleviates oversmoothing. The first is that, in principle, the adaptive message filtering scheme reduces the synchronous exchange of \textit{all} messages at a given layer, and message exchange will be different depending on the specific layer. The second is that the readout mapping of each layer directly propagates the gradient of the loss into the corresponding message passing operation, which encourages diversity of node representations of each layer $\ell$ as long as $q(\ell ; \bm{\lambda})$ is large enough (that is, layer $\ell$'s output is important for the final prediction). In our experiments, we will show that \method{} can generate architectures in which the Dirichlet energy does not decay exponentially and thus suffers less from oversmoothing than the baselines.

\begin{figure}[t]
    \centering
    \resizebox{0.5\columnwidth}{!}{\input{images/oversquashing}}
    \caption{Comparison of the 2-hop computational tree necessary to compute $\bm{h}^2_3$ in the graph of Figure \ref{fig:overview} for standard message passing (left) and \method{} (right), where we discretized message filtering to simplify the concept. \method{} can effectively prune/filter information in sub-trees to propagate only the relevant information for the task.}
    \label{fig:oversquashing}
\end{figure}

\paragraph{Oversquashing.} The term oversquashing refers to the compression of an exponentially-growing amount of information \citep{micheli_neural_2009} into fixed-size node embeddings \citep{alon_on_2021,di_over_2023}, causing a possibly severe bottleneck that hampers DGNs' ability to effectively propagate task-specific information. An intuitive visualization is provided in Figure \ref{fig:oversquashing} (left), where node $3$ of the graph defined in Figure \ref{fig:overview} needs to compress information of its 2-hop neighborhood into a single node embedding. The literature on the topic is already vast despite its very recent introduction; some works address topological bottlenecks through rewiring of the original graph structure \citep{topping_understanding_2022}, while others preserve information by viewing the message passing operations through the lens of ordinary differential equations \citep{gravina_antisymmetric_2023,heilig2025porthamiltonian}. By properly modifying the curvature of a graph \citep{topping_understanding_2022}, some graph rewiring approaches aim at increasing the \textit{sensitivity} of a node's $u$ embedding $\bm{h}_u^L$  with respect to the input $\bm{x}_v$ of another node $v$, that is $\left|\left| \frac{\partial \bm{h}_v^L}{\partial \bm{x}_u} \right|\right|_{1}$. \citet{topping_understanding_2022} argue that increasing the sensitivity can alleviate oversquashing and better capture long-range dependencies. Indeed, by rewiring two distant nodes with a new edge, the sensitivity of these two nodes will almost certainly increase. 

While we do agree that long-range dependencies can be better captured, we argue that rewiring might make the computational bottleneck problem \textit{worse} by adding extra information to be compressed into a node's embedding (assuming other edges are not removed). In contrast, the adaptive filtering scheme of \method{} shown in Figure \ref{fig:oversquashing} (right) might decrease the overall sensitivity defined above, but at the same time it will reduce the number of messages that need to be compressed into node $3$, hence \textit{alleviating} oversquashing. Similarly, the synthetic datasets defined in \citet{alon_on_2021}, which are meant to measure how well a model addresses oversquashing, require that \textit{all} information is preserved to solve a task. This would certainly be a good test-bed for ODE-based models \citep{gravina_antisymmetric_2023}, but other tasks might require propagating only a subset of the total information contained in the graph. The ability to \textit{isolate} such information from the rest can be seen as a solution to the oversquashing problem, which is exactly the opposite goal of the synthetic tasks previously mentioned. In summary, the problem of ``oversquashing'' is clearly multi-faceted and requires great care regarding its evaluation. As such, it might be a good idea for the future to decompose over-squashing into simpler sub-problems, such as the ability to isolate the relevant information (which \method{} can do), the ability to propagate all information, and the ability to increase the sensitivity between far-away nodes.

We complement this discussion with a Theorem, inspired by \citet{di_over_2023}, that shows how \method \ can control the upper bound on the sensitivity by filtering messages.

\paragraph{Theorem \ref{th:upper-bound}}   
    For \method{} with $m$ layers, embedding dimension $d$, and $u, v \in \mathcal{V}$,    
    \[
    \left\| \frac{\partial h_v^{\ell}}{\partial h_u^{0}} \right\|_{L^1} 
    \leq d \left( \left( c_{\mathrm{up}} \left( c_{\mathrm{rs}} I + c_{\mathrm{mp}} \left( c_F k_h + k_F \right) A \right) \right)^{\ell} \right)_{vu}.
    \]
    Here, MPNN is in the following form 
    \[
    (\mathbb{R}^d \ni) \ h_v^{\ell} = \mathrm{up}\left( \mathrm{rs}(h_v^{\ell-1}) + \mathrm{mp}(\sum_u A_{vu} F(h_u^{\ell-1}) \odot h_u^{\ell-1}) \right)
    \]
    where $\mathrm{up}, \mathrm{rs},$ and $\mathrm{mp}$ are Lipschitz functions as in \citet{di_over_2023} with constants $c_{\mathrm{up}}, c_{\mathrm{rs}}, c_{\mathrm{mp}}$, $c_F$ is the upper bound of the entry-wise $L^ 1$ matrix norm of $\frac{\partial F}{\partial x}$ for the filtering function $F$, $k_h$ is the maximal absolute value among the entries of $h$, and similarly $k_F$ for the output of $F$. 

    \begin{proof}
    
    \textbf{Case: $\ell = 1$}
    
    The gradient of the $\alpha$-th element of hidden vector after the first layer with respect to the $\beta$-th element of an input is written as
    \[
    \frac{\partial h_v^{1, \alpha}}{\partial h_u^{0, \beta}} 
    = \sum_{p=1}^d \frac{\partial \text{up}^{1, \alpha}}{\partial x_p} \left(\sum_{r=1}^d \frac{\partial \text{rs}^{1, p}}{\partial x_r} \frac{\partial h_v^{0, r}}{\partial h_u^{0, \beta}} + \sum_{q=1}^d \frac{\partial \text{mp}^{1, p}}{\partial x_q} \sum_{z} A_{vz} \frac{\partial \left( F(h_z^{0}) \odot h_z^{0} \right)^q}{\partial h_u^{0, \beta}} \right).
    \]
    Hereinafter, $\alpha, \beta \in [d]$ and for a scalar-valued function $f$, $\frac{\partial f}{\partial x_s}$ represents the derivative of $f$ with respect to $s$-th element $x_s$ of its vector input. We define a part of the second term in the right-hand side of the above equation as
    \[
    M(z) := \frac{\partial \left( F_0(h_z^{0}) \odot h_z^{0} \right)^q}{\partial h_u^{0, \beta}},
    \]
    and then 
    \begin{align*}
    M(z) &= \frac{\partial F^q(h_z^{0})}{\partial h_u^{0, \beta}} h_z^{0, q} + F^q(h_z^{0}) \frac{\partial h_z^{0, q}}{\partial h_u^{0, \beta}} \\
    &= \sum_{r=1}^d \frac{\partial F^q(h_z^{0})}{\partial h_z^{0, r}} \frac{\partial h_z^{0, r}}{\partial h_u^{0, \beta}} h_z^{0, q} + F^q(h_z^{0}) \frac{\partial h_z^{0, q}}{\partial h_u^{0, \beta}}.
    \end{align*}
    Therefore,
    
    \begin{align}
    \left| \frac{\partial h_v^{1, \alpha}}{\partial h_u^{0, \beta}} \right| 
    &= \left| \frac{\partial \text{up}^{1, \alpha}}{\partial x_p} \left( \sum_{r=1}^d \frac{\partial \text{rs}^{1, p}}{\partial x_r} \frac{\partial h_v^{0, r}}{\partial h_u^{0, \beta}} + \sum_{q=1}^d \frac{\partial \text{mp}^{1, p}}{\partial x_q} \sum_{z} A_{vz} M(z) \right) \right| \nonumber \\
    &\leq \left| \frac{\partial \text{up}^{1, \alpha}}{\partial x_p}  \frac{\partial \text{rs}^{1, p}}{\partial x_\beta} \delta_{v u} \right| + \left| \frac{\partial \text{up}^{1, \alpha}}{\partial x_p} \sum_{q=1}^d \frac{\partial \text{mp}^{1, p}}{\partial x_q} \sum_{z} A_{vz} M(z) \right| \label{eq:firstbd}.
    \end{align}
    
    
    The right-hand side of (\ref{eq:firstbd}) is further expanded to
    \begin{align}
    &\sum_{p=1}^d \frac{\partial \text{up}^{1, \alpha}}{\partial x_p} \sum_{q=1}^d \frac{\partial \text{mp}^{1, p}}{\partial x_q} \sum_z A_{vz} \sum_{r=1}^d \frac{\partial F^q}{\partial x_r} \frac{\partial h_z^{0, r}}{\partial h_u^{0, \beta}} h_z^{0, q} \label{eq:leftfirstbd} \\ 
    &+ \sum_{p=1}^d \frac{\partial \text{up}^{1, \alpha}}{\partial x_p} \sum_{q=1}^d \frac{\partial \text{mp}^{1, p}}{\partial x_q} \sum_z A_{vz} F^q(h_z^{0})  \frac{\partial h_z^{0,q}}{\partial h_u^{0, \beta}}. \label{eq:rightfirstbd}
    \end{align}
    Since this is the very first layer of MPNN, both terms can be reduced to simpler forms. Indeed, (\ref{eq:leftfirstbd}) can be reduced to
    \[
    (\ref{eq:leftfirstbd}) = \sum_{p=1}^d \frac{\partial \text{up}^{1, \alpha}}{\partial x_p} \sum_{q=1}^d \frac{\partial \text{mp}^{1, p}}{\partial x_q} A_{vu} \frac{\partial F^q}{\partial x_\beta} h_u^{0, q}.
    \]
    Using (generalized) Hölder's inequality, inequality (3.1) in [1], we get
    \begin{align}
    |(\ref{eq:leftfirstbd})| &\leq A_{vu} \sum_{p=1}^d \left| \frac{\partial \text{up}^{1, \alpha}}{\partial x_p} \right| \sum_{q=1}^d \left| \frac{\partial \text{mp}^{1, p}}{\partial x_q} \right| \left| \frac{\partial F^q}{\partial x_\beta} \right| \left| h_u^{0, q} \right| \nonumber \\
    &\leq A_{vu} k_h\sum_{p=1}^d \left| \frac{\partial \text{up}^{1, \alpha}}{\partial x_p} \right| \sum_{q=1}^d \left| \frac{\partial \text{mp}^{1, p}}{\partial x_q} \right| \sum_{q=1}^d \left| \frac{\partial F^q}{\partial x_\beta} \right| \nonumber \\
    &\leq A_{vu} \cdot c_{\text{up}} \cdot c_{\text{mp}} \cdot c_F \cdot k_h .
    \end{align}
    Similarly, the norm of (\ref{eq:rightfirstbd}) is also bounded from above
    \[\left| (\ref{eq:rightfirstbd}) \right| \leq A_{vu} c_{up} c_{mp} k_{F}.\]
    The left-hand term of (\ref{eq:firstbd}) is also bounded by $c_{\mathrm{up}} c_{\mathrm{rs}} I_{vu}$ (see also the proof of Theorem B.1. in \citet{di_over_2023},) and we get
    \[\left| \frac{\partial h_v^{1, \alpha}}{\partial h_u^{0, \beta}} \right| 
    \leq  c_{\mathrm{up}} \left( c_{\mathrm{rs}} I + c_{\mathrm{mp}} \left( c_F k_h + k_F \right) A \right)_{vu}.\]
    \\
    \\
    \textbf{Case: arbitrary $\ell$.}
    
    The gradient of the hidden vector at the layer of arbitrary $\ell$ is written as
    \begin{align*}
    \frac{\partial h_v^{\ell+1, \alpha}}{\partial h_u^{0, \beta}} = \sum_{p=1}^d \frac{\partial \text{up}^{\ell+1, \alpha}}{\partial x_p} \left( \sum_{r=1}^d \frac{\partial \text{rs}^{\ell+1, p}}{\partial x_r} \frac{\partial h_v^{\ell, r}}{\partial h_u^{0, \beta}} 
    + \sum_{q=1}^d \frac{\partial \text{mp}^{\ell+1, p}}{\partial x_q} \sum_{z} A_{vz} \frac{\partial \left( F(h_z^{\ell}) \odot h_z^{\ell} \right)^q}{\partial h_u^{0, \beta}} \right).
    \end{align*}
    Define and expand the right-hand term of the above equation as follows:
    \begin{align*}
    (\#) &:= \sum_{p}\frac{\partial \text{up}^{\ell+1, \alpha}}{\partial x_p} \sum_{q}\frac{\partial \text{mp}^{\ell+1, p}}{\partial x_q} \sum_{z} A_{vz} \frac{\partial \left( F(h_z^{\ell}) \odot h_z^{\ell} \right)^q}{\partial h_u^{0, \beta}} \\
    &= \sum_{p}\frac{\partial \text{up}^{\ell+1, \alpha}}{\partial x_p}  \sum_{q}\frac{\partial \text{mp}^{\ell+1, p}}{\partial x_q} \sum_{z} A_{vz} \left( \sum_{r=1}^d \frac{\partial F^q}{\partial x_r} \frac{\partial h_z^{\ell, r}}{\partial h_u^{0, \beta}} h_z^{\ell, q} + F^q(h_z^{\ell}) \frac{\partial h_z^{\ell,q}}{\partial h_u^{0, \beta}} \right).
    \end{align*}
    Then, its norm is 
    \begin{align}
    \left| (\#) \right| 
    &\leq \sum_{p} \left| \frac{\partial \text{up}^{\ell+1, \alpha}}{\partial x_p} \right| \sum_{q} \left| \frac{\partial \text{mp}^{\ell+1, p}}{\partial x_q} \right| \sum_{z} A_{vz} \sum_{r=1}^d \left| \frac{\partial F^q}{\partial x_r}\right| \left| \frac{\partial h_z^{\ell, r}}{\partial h_u^{0, \beta}}\right| \left| h_z^{\ell, q}\right| \nonumber \\
    &\ \ \ \  + \sum_{p} \left| \frac{\partial \text{up}^{\ell+1, \alpha}}{\partial x_p} \right| \sum_{q} \left| \frac{\partial \text{mp}^{\ell+1, p}}{\partial x_q} \right| \sum_{z} A_{vz}  \left| F^q(h_z^{\ell})  \right| \left|\frac{\partial h_z^{\ell,q}}{\partial h_u^{0, \beta}}  \right| \nonumber  \\
    & \leq \Big( \sum_{p} \left| \frac{\partial \text{up}^{\ell+1, \alpha}}{\partial x_p} \right| \sum_{q} \left| \frac{\partial \text{mp}^{\ell+1, p}}{\partial x_q} \right| \sum_{z} A_{vz} \sum_{r=1}^d \left| \frac{\partial F^q}{\partial x_r}\right| \left| h_z^{\ell, q}\right| \nonumber \\
    &\quad \quad \quad \quad  + \sum_{p} \left| \frac{\partial \text{up}^{\ell+1, \alpha}}{\partial x_p} \right| \sum_{q} \left| \frac{\partial \text{mp}^{\ell+1, p}}{\partial x_q} \right| \sum_{z} A_{vz}  \left| F^q(h_z^{\ell})  \right| \Big) \nonumber \\
    &\quad \quad \quad \quad \quad \quad \quad \quad \times c_{\mathrm{up}}^{\ell+1} (\left( c_{\mathrm{rs}} I + c_{\mathrm{mp}} \left( c_F k_h + k_F \right) A \right)^{\ell} )_{vu}. \label{eq:generalbd}
    \end{align}
    The first term of (\ref{eq:generalbd}), except the constant term, is bounded from above:
    \begin{align}
    &\sum_{p=1}^d \left| \frac{\partial \text{up}^{\ell+1, \alpha}}{\partial x_p} \right|
    \sum_z A_{vz}
    \left( \sum_{q=1}^d \left| \frac{\partial \text{mp}^{\ell+1, p}}{\partial x_q} \right|
    \left| h_z^{\ell, q} \right| 
    \sum_{r} \left| \frac{\partial F^q}{\partial x_r} \right| \right) \nonumber \\ 
    &\leq \sum_{p=1}^d \left| \frac{\partial \text{up}^{\ell+1, \alpha}}{\partial x_p} \right|
    \sum_z A_{vz}
    \left( \sum_{q=1}^d \left| \frac{\partial \text{mp}^{\ell+1, p}}{\partial x_q} \right| \right)
    k_h
    \left( \sum_{r, q} \left| \frac{\partial F^q}{\partial x_r} \right| \right) \nonumber \\
    &\leq c_{\text{up}} \cdot c_{\text{mp}} \cdot k_h \cdot c_F \cdot \sum_z A_{vz}. \label{eq:leftgeneralbd}
    \end{align}
    Note that the first inequality is derived by using H\"{o}lder inequality again.
    Similarly, the second term of (\ref{eq:generalbd}) is 
    \begin{align}
    \sum_{p=1}^d \left| \frac{\partial \text{up}^{\ell+1, \alpha}}{\partial x_p} \right|
    \sum_z A_{vz} \sum_{q=1}^d \left| \frac{\partial \text{mp}^{\ell+1, p}}{\partial x_q} \right|
    \left| F(h_z^{\ell, q}) \right| \nonumber
    &\leq \sum_{p=1}^d \left| \frac{\partial \text{up}^{\ell+1, \alpha}}{\partial x_p} \right|
    \sum_z A_{vz}
    \left( \sum_{q=1}^d \left| \frac{\partial \text{mp}^{\ell+1, p}}{\partial x_q} \right| k_F \right) \nonumber \\
    &\leq c_{\text{up}} \cdot c_{\text{mp}} \cdot k_F \cdot \sum_z A_{vz}. \label{eq:rightgeneralbd}
    \end{align}
    
    Finally, define
    \[
    \text{UB}(\ell) := c_{\text{up}}^{\ell} (\widetilde{UB})^{\ell} := c_{\text{up}}^{\ell} \left( c_{\text{rs}} I + c_{\text{mp}} \left( c_F k_h + k_F \right) A \right)^{\ell},
    \]
    and then the norm of $\frac{\partial h_v^{\ell+1, \alpha}}{\partial h_u^{0, \beta}} $ is 
    
    \begin{align*}
    \left|\frac{\partial h_v^{\ell+1, \alpha}}{\partial h_u^{0, \beta}} \right|
    &\leq \sum_{p=1}^d \left| \frac{\partial \text{up}^{\ell+1, \alpha}}{\partial x_p} \right|
    \left( \sum_{r=1}^d \left| \frac{\partial \text{rs}^{\ell+1, p}}{\partial x_r} \right| \left| \frac{\partial h_v^{\ell, r}}{\partial h_u^{0, \beta}} \right| \right) + ((\ref{eq:leftgeneralbd}) + (\ref{eq:rightgeneralbd})) \cdot \text{UB}(\ell)_{zu} \\
    &\leq c_{\text{up}} c_{\text{rs}} \cdot \text{UB}(\ell)_{vu} 
    + c_{\text{up}} c_{\text{mp}} k_h c_F \sum_z A_{vz} \text{UB}(\ell)_{zu}
    + c_{\text{up}} c_{\text{mp}} k_F \sum_z A_{vz} \text{UB}(\ell)_{zu} \\
    &= c_{\text{up}}^{\ell+1} \left( c_{\text{rs}} \cdot \widetilde{UB}^{\ell}_{vu} 
    + \sum_z c_{\text{mp}} (k_h c_F + k_F) A_{vz} \widetilde{UB}^{\ell}_{zu} \right) \\
    &= c_{\text{up}}^{\ell+1} \left(\left( c_{\text{rs}} I + c_{\text{mp}} (k_h c_F + k_F) A \right) \widetilde{UB}^{\ell} \right)_{vu}\\
    &= \text{UB}(\ell+1)_{vu},
    \end{align*}
    which completes the proof.
    \end{proof}

\paragraph{Underreaching.} Finally, underreaching is defined as the inability of standard message passing with $K$ layers to capture interactions of range greater than $K$. \citet{alon_on_2021} addresses this problem by adding a message passing layer on a fully connected graph at the last layer of the architecture, which empirically improves the performances but does not fundamentally solve the problem. A solution to this problem is letting the model decide the right depth of the architecture for the task, which is exactly what \method{} does. 

Below, we provide the extended formulation and proof of our theorem about the propagation of a message unchanged through the graph.

\paragraph{Theorem \ref{th:reachability}}
\textit{Let us consider a graph $g=(\mathcal{V}, \mathcal{E}, \mathcal{X})$, where $\mathcal{V}=\{1,\dots,n_g\}$ is the set of nodes connected via oriented edges $\mathcal{E}=\{(u,v) | u,v \in \mathcal{V}\}$, and $\mathcal{X}=\{\mathbf{x}_v \in \mathbb{R}^d | v \in \mathcal{V}\}$ is the set of $d$ node attributes. The neighborhood of a node $v$ is defined as $\mathcal{N}_v = \{ u | (u,v) \in \mathcal{E}\}$. Further, let us define the following message aggregation scheme, which produces vectors $\bm{h}^\ell_v \in \mathbb{R}^d$ for node $v$ at iteration $\ell$, up to a maximum iteration $K$:
}    \begin{align}
        & \bm{h}_{v}^{0} = \bm{x}_v \\ \nonumber
        & \bm{h}_{v}^{\ell} = \sum_{u \in \mathcal{N}_{v}}\bm{F}(u,\ell) \odot \bm{h}_{u}^{\ell-1} \ \ \ (\ell \geq 1), \nonumber
    \end{align}
\textit{where $\bm{F}(u,\ell) \in (0,1)^d, \forall u \in \mathcal{V},\ell \in[1,K]$ and $\odot$ is the element-wise product. Let us also assume that $g$
contains two (not necessarily distinct) nodes $v$ and $u$, and a walk $\left((v,v_2),\dots,(v_{K},u)\right)$ of length $K\geq1$ exists between them. 
}
\textit{Then, $\forall \epsilon > 0$ there exists a parametrization of $\bm{F}$ such that $\bm{h}^K_{u}$ belongs to the closed ball $\mathcal{B}(\bm{x}_{v}, d\epsilon, ||\cdot||_1)$ of radius $d\epsilon$ centered at $\bm{x}_{v}$ under norm $||\cdot||_1$.}    
\begin{proof}
    We prove the statement by induction over $K$.
    \paragraph{Base case ($K=1$):} In this case, the source node $v$ is a neighbor of the destination node $u$. It holds that 
    \begin{align}
        \bm{h}_{u}^{1} & = \sum_{u' \in \mathcal{N}_{u}}\bm{F}(u',1) \odot \bm{h}_{u'}^{0} \nonumber \\ 
        & = \bm{F}(v,1) \odot \bm{x}_{v} + \sum_{u' \in \mathcal{N}_{u}\setminus \{v\} }\bm{F}(u',1) \odot \bm{x}_{u'}.  \nonumber
    \end{align}
    We choose $\bm{F}(v,1) = \bm{1} - \bm{\varepsilon}_v$ and $\bm{F}(u',1) = \bm{\varepsilon}_{u'}$, for some $\bm{\varepsilon}_v, \bm{\varepsilon}_{u'}$ with all components $> 0$, therefore
    \begin{align}
        \bm{h}_{u}^{1} & = (\bm{1} - \bm{\varepsilon}_{v}) \odot \bm{x}_{v} + \sum_{u' \in \mathcal{N}_{u}\setminus \{v\} } \bm{\varepsilon}_{u'} \odot \bm{x}_{u'}  \nonumber \\
        & = \bm{x}_{v} + \sum_{u' \in \mathcal{N}_{u}} (-1)^{\delta(u',v)}\bm{\varepsilon}_{u'} \odot \bm{x}_{u'}. \nonumber
    \end{align}
    At this point, we can always choose $\bm{\varepsilon}_{u'}, \forall u' \in \mathcal{N}_{u}$ such that, $\forall i \in [1,\dots,d]$ it holds 
    \begin{align}
        - \epsilon \leq \sum_{u' \in \mathcal{N}_{u}} (-1)^{\delta(u',v)}\varepsilon_{u'}[i] x_{u'}[i] \leq \epsilon \nonumber
    \end{align}
    noting that the inequality is due to the fact that values of $\bm{x}_{u'}$ can be negative. Therefore, for every $i$ we have that 
    \begin{align}
        x_{v}[i] - \epsilon \leq h_{u}^{1}[i] \leq x_{v}[i] + \epsilon \equiv |h_{u}^{1}[i] - x_{v}[i]| \leq \epsilon
    \end{align}
    which implies $\bm{h}^1_{u} \in \mathcal{B}(\bm{x}_{v}, d\epsilon, ||\cdot||_1)$. 

    \paragraph{Inductive case ($K>1$):} Let us consider a random walk of length $K$, and assume our proposition holds for values up to $K-1$. With the same arguments as before, we can write
        \begin{align}
        \bm{h}_{u}^{K} & = (\bm{1} - \bm{\varepsilon}_{v_K}) \odot \bm{h}^{K-1}_{v_K} + \sum_{u' \in \mathcal{N}_{u}\setminus \{v_K\}} \bm{\varepsilon}_{u'} \odot \bm{h}^{K-1}_{u'}  \nonumber \\
        & = \bm{h}^{K-1}_{v_K} + \sum_{u' \in \mathcal{N}_{u}} (-1)^{\delta(u',v_K)}\bm{\varepsilon}_{u'} \odot \bm{h}^{K-1}_{u'}. \nonumber
    \end{align}
    By the inductive hypothesis we know that there exists a parametrization of $\bm{F}$ (up to iteration $K-1$) such that $\bm{h}^{K-1}_{v_K} \in \mathcal{B}(\bm{x}_{v}, d\frac{\epsilon}{2}, ||\cdot||_1)$. As before, we pick $\bm{\varepsilon}_{u'}, \forall u' \in \mathcal{N}_{u}$ such that, $\forall i \in [1,\dots,d]$ it holds  
        \begin{align}
        - \frac{\epsilon}{2} \leq \sum_{u' \in \mathcal{N}_{u}} (-1)^{\delta(u',v_K)}\varepsilon_{u'}[i] h^{K-1}_{u'}[i] \leq \frac{\epsilon}{2} \nonumber
    \end{align}
    and therefore, noting that  $x_{v}[i] - \frac{\epsilon}{2} \leq h_{u}^{K-1}[i] \leq x_{v}[i] + \frac{\epsilon}{2}$ by the inductive hypothesis,
        \begin{align}
        x_{v}[i] - 2*\frac{\epsilon}{2} \leq h_{u}^{K}[i] \leq x_{v}[i] +2*\frac{\epsilon}{2} \equiv |h_{u}^{K}[i] - x_{v}[i]| \leq \epsilon
    \end{align}
    meaning $\bm{h}^K_{u} \in \mathcal{B}(\bm{x}_{v}, d\epsilon, ||\cdot||_1)$.
\end{proof}

It is worth noting that the use of a differentiable sigmoidal activation makes it impossible to propagate the exact same value, but this holds in the limit of $\epsilon\rightarrow 0$. In light of the above discussion, this theorem hints at \method{} indeed mitigating oversmoothing, oversquashing, and underreaching by being able to propagate a single message unchanged, which is reminiscent of asynchronous message passing \citep{faber_asynchronous_2023}.

\begin{figure}[ht]
    \centering
    \resizebox{0.6\columnwidth}{!}{\input{images/theorem}}
    \caption{We sketch the idea of the proof of Theorem \ref{th:reachability} for a graph of 4 nodes. When messages are filtered appropriately, a message can flow between node 1 and node 4 almost unchanged.}
    \label{fig:theorem-sketch}
\end{figure}

\section{Hyper-parameters Details}
\label{sec:hyper-params}

The set of hyper-parameters tried for the baselines and for \method{} (except for the number of layers) is the same of \citet{gravina_antisymmetric_2023} and \citet{tonshoff_did_2023}. 

In particular, for the baselines and AMP on the synthetic tasks, we used an Adam optimizer with learning rate $0.003$, weight decay $1e-6$, embedding dimension in $[10,20,30]$, tanh activation function and a number of layers (except for \method{}) in $[1,5,10,20]$.

Instead, results on the peptides-func dataset rely on the following hyper-parameters: learning rate $0.001$, dropout $0.1$, $6,8,10$ layers for GCN, GINE and GatedGCN respectively, and similarly embedding dimensions $235,160$ and $95$. We use a readout with a depth of $3$ layers, and RWSE positional encodings, a batch size of 200 and 250 maximum training epochs.

Finally, results on the peptides-struct rely on a similar set of hyper-parameters with the following exceptions: number of layers $6,10,8$ for GCN, GINE, and GatedGCN respectively, Laplacian positional encodings, and embedding dimensions $235,145,100$.

Below, we report the additional hyper-parameters values that we introduced when evaluating \method{}.

\paragraph{Synthetic Datasets} To perform the grid search on \method{}, in addition to the hyper-parameters range used for the base methods (with the exception of the depth), we tested four different distributions $q(\ell; \bm{\lambda})$: a Poisson with initial rate $\lambda=10$, an FN with initial parameters $\mu=10$ and $\sigma \in \{5, 10\}$, and a mixture of two folded normal distributions with initial parameters $\mu_1=5, \sigma_1=3, \mu_2=15, \sigma_2=3$. We fix the prior $p(\bm{\theta^\ell})=\mathcal{N}(\bm{\theta^\ell}; \bm{0}, 10*\textbf{I})$, and we choose between three priors $p(L)$: an uninformative prior, a Poisson with rate $5$, and a folded normal with parameters $\mu=5$ and $\sigma=10$. Finally, the message filtering function was chosen between one that does not filter at all, a function $f(\bm{x})$ acting on node features, and a function $f(\bm{h^{\ell}})$ acting on node embeddings.

\paragraph{Chemical Datasets} We tested three different distributions $q(\ell; \bm{\lambda})$: a Poisson with the initial rate $\lambda=5$, a folded normal with initial parameters $\mu=5$ and $\sigma =1$, and a mixture of two folded normal distributions with initial parameters $\mu_1=1, \sigma_1=1, \mu_2=5, \sigma_2=1$. The tested message filtering functions and the number of layers are the same as the synthetic tasks, whereas $p(\bm{\theta^\ell})=\mathcal{N}(\bm{\theta^\ell}; \bm{0}, 10*\textbf{I})$ for peptides-func and $p(\bm{\theta^\ell})=\mathcal{N}(\bm{\theta^\ell}; \bm{0}, 1*\textbf{I})$ for peptides-struct. Also, following \citet{tonshoff_did_2023}, we use dropout in the output layers, after each activation function, and set its value to $0.1$.

\section{Node Classification Results}
\label{sec:node-classification}

To demonstrate \method{}'s applicability on a broader range of tasks, we report node classification performances on different datasets with varying degrees of homophily, following the data split strategy of \citet{zhu_beyond_2020}. We train a $GCN$ by testing the following hyper-parameters: embedding size in $[8,32,64]$ and number of layers in $[2,4,6]$, Adam optimizer with learning rate $0.01$, a maximum of 2000 training epochs and early stopping with patience $250$ on the validation performance. The hyper-parameters were the same for $\method{}_{GCN}$ with the exception for the number of layers, which is replaced by the distributions tried for the chemical datasets.

\setlength{\tabcolsep}{3pt}
\begin{table}[ht]
    \centering
    \small
    \begin{tabular}{lccccc  ccc} %
    \toprule
       &  \texttt{\bf Texas}           &   \texttt{\bf Wisconsin}           &   \texttt{\bf Actor}            &   \texttt{\bf Squirrel}   &   \texttt{\bf Chameleon}  &   \texttt{\bf Citeseer}           &   \texttt{\bf Pubmed}            &   \texttt{\bf Cora} \\
          \textbf{Hom.\ ratio} $h$ & \textbf{0.11} & \textbf{0.21} & \textbf{0.22} & \textbf{0.22} & \textbf{0.23}  & \textbf{0.74} & \textbf{0.8} & \textbf{0.81} \\
		\textbf{\#Nodes} & 183 & 251 & 7,600 & 5,201 & 2,277 & 3,327 & 19,717 & 2,708 \\
		\textbf{\#Edges} & 295 & 466 & 26,752 & 198,493 & 31,421 &  4,676 & 44,327 & 5,278 \\
		\textbf{\#Classes} & 5 & 5 & 5 & 5 & 5 & 7 & 3 & 6 \\ 
    \midrule
    \midrule
       {H$_2$GCN-1} & $84.86{\scriptstyle\pm6.77}$ & $86.67{\scriptstyle\pm4.69}$ & $35.86{\scriptstyle\pm1.03}$ & $36.42{\scriptstyle\pm1.89}$ & $57.11{\scriptstyle\pm1.58}$ & $77.07{\scriptstyle\pm1.64}$ & $89.40{\scriptstyle\pm0.34}$ & $86.92{\scriptstyle\pm1.37}$ \\
       {H$_2$GCN-2} & $82.16{\scriptstyle\pm5.28}$ & $85.88{\scriptstyle\pm4.22}$ & $35.62{\scriptstyle\pm1.30}$ & $37.90{\scriptstyle\pm2.02}$ & $59.39{\scriptstyle\pm1.98}$ & $76.88{\scriptstyle\pm1.77}$ & $89.59{\scriptstyle\pm0.33}$ & $87.81{\scriptstyle\pm1.35}$ \\
       {GraphSAGE} & $82.43{\scriptstyle\pm6.14}$ & $81.18{\scriptstyle\pm5.56}$ & $34.23{\scriptstyle\pm0.99}$ & $41.61{\scriptstyle\pm0.74}$ & $58.73{\scriptstyle\pm1.68}$ & $76.04{\scriptstyle\pm1.30}$ & $88.45{\scriptstyle\pm0.50}$ & $86.90{\scriptstyle\pm1.04}$ \\
       {GCN-Cheby} & $77.30{\scriptstyle\pm4.07}$ & $79.41{\scriptstyle\pm4.46}$ & $34.11{\scriptstyle\pm1.09}$ & $43.86{\scriptstyle\pm1.64}$ & $55.24{\scriptstyle\pm2.76}$ & $75.82{\scriptstyle\pm1.53}$ & $88.72{\scriptstyle\pm0.55}$ & $86.76{\scriptstyle\pm0.95}$  \\
	   {MixHop} & $77.84{\scriptstyle\pm7.73}$ & $75.88{\scriptstyle\pm4.90}$ & $32.22{\scriptstyle\pm2.34}$ & $43.80{\scriptstyle\pm1.48}$ & $60.50{\scriptstyle\pm2.53}$ & $76.26{\scriptstyle\pm1.33}$ & $85.31{\scriptstyle\pm0.61}$ & $87.61{\scriptstyle\pm0.85}$ \\ 
        \midrule
        {GraphSAGE+JK} & $83.78{\scriptstyle\pm2.21}$ & $81.96{\scriptstyle\pm4.96}$ & $34.28{\scriptstyle\pm1.01}$ & $40.85{\scriptstyle\pm1.29}$ & $58.11{\scriptstyle\pm1.97}$ $76.05{\scriptstyle\pm1.37}$ & $88.34{\scriptstyle\pm0.62}$ & $85.96{\scriptstyle\pm0.83}$ \\
        {Cheby+JK} & $78.38{\scriptstyle\pm6.37}$ & $82.55{\scriptstyle\pm4.57}$ & $35.14{\scriptstyle\pm1.37}$ & $45.03{\scriptstyle\pm1.73}$ & $63.79{\scriptstyle\pm2.27}$ & $74.98{\scriptstyle\pm1.18}$ & $89.07{\scriptstyle\pm0.30}$ & $85.49{\scriptstyle\pm1.27}$ \\
        {GCN+JK} & $66.49{\scriptstyle\pm6.64}$ & $74.31{\scriptstyle\pm6.43}$ & $34.18{\scriptstyle\pm0.85}$ & $40.45{\scriptstyle\pm1.61}$ & $63.42{\scriptstyle\pm2.00}$  $74.51{\scriptstyle\pm1.75}$ & $88.41{\scriptstyle\pm0.45}$ & $85.79{\scriptstyle\pm0.92}$ \\
       \midrule
	   {GAT} & $58.38{\scriptstyle\pm4.45}$ & $55.29{\scriptstyle\pm8.71}$ & $26.28{\scriptstyle\pm1.73}$ & $30.62{\scriptstyle\pm2.11}$ & $54.69{\scriptstyle\pm1.95}$ & $75.46{\scriptstyle\pm1.72}$ & $84.68{\scriptstyle\pm0.44}$ & $82.68{\scriptstyle\pm1.80}$  \\
	   {GEOM-GCN*} & $67.57$ & $64.12$ & $31.63$ & $38.14$ & $60.90$ {$77.99$} &  {$90.05$} & $85.27$  \\
	   \midrule       
	   {MLP} & $81.89{\scriptstyle\pm4.78}$ & $85.29{\scriptstyle\pm3.61}$ & $35.76{\scriptstyle\pm0.98}$ & $29.68{\scriptstyle\pm1.81}$ & $46.36{\scriptstyle\pm2.52}$ & $72.41{\scriptstyle\pm2.18}$ & $86.65{\scriptstyle\pm0.35}$ & $74.75{\scriptstyle\pm2.22}$  \\
       \midrule
       {GCN} & $52.73{\scriptstyle\pm5.98}$ & $44.90{\scriptstyle\pm7.38}$  & $28.02{\scriptstyle\pm0.59}$ & $27.73{\scriptstyle\pm0.91}$ & $41.89{\scriptstyle\pm1.83}$ & $72.71{\scriptstyle\pm1.54}$ & $87.62{\scriptstyle\pm0.59}$ & $84.24{\scriptstyle\pm1.44}$ \\
       {\method{}$_{GCN}$} & $81.46{\scriptstyle\pm3.06}$ & $80.45{\scriptstyle\pm4.47}$ & $34.14{\scriptstyle\pm0.71}$ & $35.28{\scriptstyle\pm0.95}$ & $48.56{\scriptstyle\pm1.70}$ & $75.07{\scriptstyle\pm1.29}$ & $89.67{\scriptstyle\pm1.29}$ & $85.68{\scriptstyle\pm1.52}$ \\
	   \bottomrule
    \end{tabular}
    \caption{Node classification results on heterophilic and homophilic graphs. Results are taken from \citet{zhu_beyond_2020}.}
    \label{tab:node-class}
\end{table}

Table \ref{tab:node-class} clearly shows that wrapping our framework around a $GCN$ always grants a performance improvement. Only for reference, we report other results taken by \cite{zhu_beyond_2020}.  However, please note that there an homophilic/heterophilic graph does not imply that the task is long-range, nor the converse is necessarily true.  

\section{Analysis of \method{}'s Predictions}
\label{sec:results-predictions}

We further delve into the predictions of \method{AMP} on a few representative cases. In Figure~\ref{fig:preds-analysis} (left), we report the mean predictions of the best performing GCN and \method{}$_{\textsc{GCN}}$ runs on the Diameter dataset, where the shaded bands denote the minimal and maximal errors that both models make. Similarly, Figure~\ref{fig:preds-analysis}~(right) shows the same plot but for ADGN and \method{}$_{\textsc{ADGN}}$ on Eccentricity. We can see how \method{} generates an almost ideal average prediction on Diameter and is able to deal with higher eccentricity than the base model (despite an almost identical error being achieved in the latter case). 

\begin{figure*}[ht]
\centering
    \begin{subfigure}
        {\includegraphics{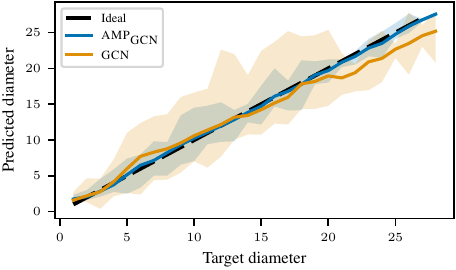}}
    \end{subfigure}
    \begin{subfigure}
        {\includegraphics{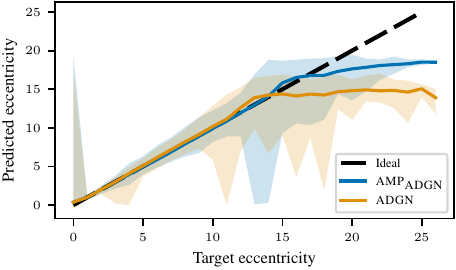}}
    \end{subfigure}
    \caption{We report the average predicted \textit{vs} target graph diameter and node eccentricities for the GCN and ADGN message passing architectures, respectively. Shaded bands denote minimal and maximal errors for each prediction.}
    \label{fig:preds-analysis}
\end{figure*}

\section{Tuned Depth of Base Models}
\label{sec:base-models-depth}

We report, for the base architectures we have tested within \method{}, the number of layers selected by the hyper-parameter search in the original papers. For the synthetic datasets, we obtained this information directly from the authors \citep{gravina_antisymmetric_2023}, whereas for the chemical datasets this information was already available in \citet{tonshoff_did_2023}.

\begin{table}[ht]
\centering
\begin{tabular}{lccccc} \toprule
                                   & Diameter & SSSP & Eccentricity & peptides-func & peptides-struct \\ \midrule
\textsc{GCN}      & 5        & 5    & 10           & 6             & 6               \\
\textsc{GIN}      & 1        & 1    & 1            & -             & -               \\
\textsc{ADGN}     & 10       & 20   & 20           & -             & -               \\
\textsc{GINE}     & -        & -    & -            & 8             & 10              \\
\textsc{GatedGCN} & -        & -    & -            & 10            & 8               \\ \bottomrule
\end{tabular}
\caption{Best number of layers selected during hyper-parameter tuning for the models considered within \method{}. These values are taken from the original papers or provided by their authors.}
\label{tab:best-layers-base-models}
\end{table}

\section{Impact of Positional and Structural Encodings}
\label{appendix:ablation-pe-se}

\begin{table}[ht]
\centering
\scriptsize
\begin{tabular}{lcc}
\toprule
Method     & peptides-func & peptides-struct \\
           & Test AP $\uparrow$ & Test MAE $\downarrow$ \\
\midrule
\multirow{6}*{\rotatebox{90}{\textsc{With PE/SE}}}
\ \ \textsc{GCN}         & $0.6860 \pm 0.0050$ & $0.2460 \pm 0.0007$ \\
\ \ \ \ \ \ \textsc{GINE}        & $0.6621 \pm 0.0067$ & $0.2473 \pm 0.0017$ \\
\ \ \ \ \ \ \textsc{GatedGCN}    & $0.6765 \pm 0.0047$ & $0.2477 \pm 0.0009$ \\
\ \ \ \ \ \ \method{}$_{\textsc{GCN}}$ & $0.7161 \pm 0.0047$ & $0.2446 \pm 0.0026$ \\
\ \ \ \ \ \ \method{}$_{\textsc{GINE}}$ & $0.7065 \pm 0.0105 $ & $0.2468 \pm 0.0026$ \\
\ \ \ \ \ \ \method{}$_{\textsc{GatedGCN}}$ &  $0.6943 \pm 0.0046 $ & $0.2480\pm 0.0012$\\
\midrule
\multirow{3}*{\rotatebox{90}{\textsc{W/O}}}
\ \ \method{}$_{\textsc{GCN}}$ & $0.7102 \pm 0.0074$ & $0.2524 \pm 0.0035$ \\
\ \ \ \ \ \ \method{}$_{\textsc{GINE}}$ & $0.6994 \pm 0.0098 $ & $0.2477 \pm 0.0017$ \\
\ \ \ \ \ \ \method{}$_{\textsc{GatedGCN}}$ &  $0.6725 \pm 0.0113 $ & $0.2466\pm 0.0008$\\

\bottomrule
\end{tabular}
\caption{Mean test scores and standard deviation averaged over 4 final runs on the chemical datasets.}
\label{tab:ablation-pe-se}
\end{table}

For completeness, we investigate the impact of the positional and structural encodings on performances. Our findings show that, on peptides-func, the performance gain compared to AMP is marginal on GCN and GINE, but it becomes substantial on GatedGCN. In all cases, even without the additional encodings, \method{} has better or comparable performance than the base models. We observe a similar trend on peptides-struct, with \method{}$_{\textsc{GCN}}$ suffering the most from the absence of extra information; here, the result is significantly worse than the base model.

\section{Ablation Study on the Depth}
\label{appendix:ablation-depth}
Performing an ablation study about the depth of \method{} implies that we perform model selection across a wide range of (fixed) depths and we learn the importance of each layer. This is an a-posteriori analysis, meaning that we the range of layers to try is suggested by our previous results; the whole point of our contribution is that it may be difficult to find the exact range of message-passing layers that the task needs. The goal of this section is to understand if it is worth fixing the depth after a sensible range of layers has been found by AMP. 

We use the information from Figure \ref{fig:depth-analysis} to set up a reasonable range of fixed number of message passing layers to try, and then we ran the experiments again on the real-world chemical datasets. We also fix the other hyper-parameters to the best configuration found, for each model and dataset, by our model selection procedure.

\begin{table}[ht]
\centering
\scriptsize
\begin{tabular}{lcc}
\toprule
Method     & peptides-func & peptides-struct \\
           & Test AP $\uparrow$ & Test MAE $\downarrow$ \\
\midrule
\ \ \ \ \ \ \ \ \ \method{}$_{\textsc{GCN}}$ & $0.7161 \pm 0.0047$ & $0.2446 \pm 0.0026$ \\
\ \ \ \ \ \ \ \ \ \method{}$_{\textsc{GINE}}$ & $0.7065 \pm 0.0105 $ & $0.2468 \pm 0.0026$ \\
\ \ \ \ \ \ \ \ \ \method{}$_{\textsc{GatedGCN}}$ &  $0.6943 \pm 0.0046 $ & $0.2480\pm 0.0012$\\
\midrule
\multirow{3}*{\rotatebox{90}{\shortstack{\textsc{Fixed} \\ \textsc{Depth}}}}
\ \method{}$_{\textsc{GCN}}$ & $0.7076 \pm 0.0059$ & $0.2497 \pm 0.0009$ \\
\ \ \ \ \ \ \ \ \ \method{}$_{\textsc{GINE}}$ & $0.6999 \pm 0.0041 $ & $0.2481 \pm 0.0014$ \\
\ \ \ \ \ \ \ \ \ \method{}$_{\textsc{GatedGCN}}$ &  $0.6750 \pm 0.0029 $ & $0.2493\pm 0.0013$\\
\bottomrule
\end{tabular}
\caption{Mean test scores and standard deviation averaged over 4 final runs on the chemical datasets.}
\label{tab:ablation-depth}
\end{table}

It appears that fixing the depth does not allow to obtain better performances than the fully adaptive \method{}.


\end{document}